\documentclass{article}

     \usepackage[final]{neurips_2021} 

\usepackage[utf8]{inputenc} 
\usepackage[T1]{fontenc}    
\usepackage{hyperref}       
\usepackage{url}            
\usepackage{booktabs}       
\usepackage{amsfonts}       
\usepackage{nicefrac}       
\usepackage{microtype}      
\usepackage{xcolor}         
\usepackage{graphicx}
\usepackage{caption}
\usepackage{subfigure}
\usepackage{booktabs} 
\usepackage{wrapfig}
\usepackage{multirow}
\usepackage{wrapfig}
\usepackage{enumitem}
\usepackage{algorithm}
\usepackage{algorithmic}
\usepackage{eqnarray}
\usepackage{makecell}
\usepackage{amssymb}
\usepackage{amsmath}
\usepackage{amsthm}
\newtheorem{theorem}{Theorem}
\newtheorem{lemma}[theorem]{Lemma}

\newtheorem{corollary}{Corollary}[theorem]
\newtheorem{definition}{Definition}

\newcommand{\EB}{\ensuremath{E_{S\cup T}}}
\newcommand{\ES}{\ensuremath{E_{S}}}
\newcommand{\ET}{\ensuremath{E_{T}}}

\newcommand{\JH}[1]{\textcolor{magenta}{JH: #1}}
\newcommand{\AM}[1]{\textcolor{cyan}{AM: #1}}

\newtheoremstyle{TheoremRep}
        {\topsep}{\topsep}              
        {\itshape}                      
        {}                              
        {\bfseries}                     
        {.}                             
        { }                             
        {\thmname{#1}\thmnote{ \bfseries #3}}
\theoremstyle{TheoremRep}
\newtheorem{theoremrep}{Theorem}

\newtheorem{corollaryrep}[theoremrep]{Corollary}

\title{Understanding the Limits of Unsupervised Domain Adaptation via Data Poisoning}

\author{Akshay Mehra\textsuperscript{1}, Bhavya Kailkhura\textsuperscript{2}, Pin-Yu Chen\textsuperscript{3} and Jihun Hamm\textsuperscript{1}\\
{\small \textsuperscript{1}Tulane University \quad \textsuperscript{2}Lawrence Livermore National Laboratory \quad \textsuperscript{3}IBM Research}\\ 
{\tt\small\{amehra, jhamm3\}@tulane.edu, kailkhura1@llnl.gov, pin-yu.chen@ibm.com}\\
}

\begin{document}

\maketitle
\begin{abstract}
Unsupervised domain adaptation (UDA) enables cross-domain learning without target domain labels by transferring knowledge from a labeled source domain whose distribution differs from that of the target.
However, UDA is not always successful and several accounts of `negative transfer' have been reported in the literature.
In this work, we prove a simple lower bound on the target domain error that complements the existing upper bound.
Our bound shows the insufficiency of minimizing source domain error and marginal distribution mismatch for a guaranteed reduction in the target domain error, due to the possible increase of induced labeling function mismatch. 
This insufficiency is further illustrated through simple distributions for which the same UDA approach succeeds, fails, and may succeed or fail with an equal chance. 
Motivated from this, we propose novel data poisoning attacks to fool UDA methods into learning representations that produce large target domain errors.  
We evaluate the effect of these attacks on popular UDA methods using benchmark datasets where they have been previously shown to be successful.
Our results show that poisoning can significantly decrease the target domain accuracy, dropping it to 
almost 0\% in some cases, with the addition of only 10\% poisoned data in the source domain. 
The failure of these UDA methods demonstrates their limitations 
at guaranteeing cross-domain generalization consistent
with our lower bound. 
Thus, evaluating UDA methods in adversarial settings such as data poisoning provides a better sense of their robustness to data distributions unfavorable for UDA.
\end{abstract}

\section{Introduction}
The problem of domain adaptation (DA) arises when the training and the test data distributions are different, violating the common assumption of supervised learning. 
In this paper, we focus on unsupervised DA (UDA), which is a special case of DA when no labeled information from the target domain is available.
This setting is useful for applications where obtaining large-scale well-curated datasets is both time-consuming and costly.  
The seminal works \cite{ben2010theory, ben2007analysis} proved an upper bound on a classifier's target domain error in the UDA setting leading to several algorithms for learning in this setting. 
Many of these algorithms rely on learning a domain invariant representation by minimizing the error on the source domain and a divergence measure between the marginal feature distributions of the source and target domains. Popular divergence measures include total variation distance, Jensen-Shannon divergence \cite{ganin2016domain,tzeng2017adversarial,zhao2018adversarial}, Wasserstein distance \cite{shen2018wasserstein,lee2017minimax,courty2017joint}, and maximum mean discrepancy \cite{long2014transfer,long2015learning,long2016unsupervised}. 
The success of these algorithms is argued in terms of minimization of the upper bound proposed in \cite{ben2010theory} along with an improvement in the target domain accuracy on benchmark UDA tasks.

Despite this success on benchmark datasets, some works \cite{ganin2016domain,liu2019transferable,zhao2019learning,wang2019characterizing} have presented evidence of the failure of these methods in different scenarios.
Recent works have explained this apparent failure of UDA methods by proposing new upper bounds \cite{zhao2019learning, combes2020domain, johansson2019support, wu2020representation} on the target domain error while others have demonstrated the cause of failure using experiments showing that learning a domain invariant representation that minimizes the source domain error can cause an increase in the error of the ideal joint hypothesis \cite{liu2019transferable}.
To provably explain this apparent failure of learning in the UDA setting we propose a lower bound on the target domain error. 
Our lower bound provides a necessary condition for successful learning in the UDA setting, complementing the existing upper bound of \cite{ben2010theory} 
and is dependent
on the difference between the labeling functions of source and target domain data induced by the representation map.
%
For cases where the induced labeling functions match on the source and the target domain data (i.e., favorable case), the success of UDA is explained using the upper bounds proposed by previous works \cite{ben2010theory, zhao2019learning, johansson2019support}.
For a representation that aligns the source and the target domain data and minimizes the error on the source but induces labeling functions that don't agree on the source and the target domain data (i.e., unfavorable case), our lower bound explains the failure of UDA.
Our analysis brings to light yet another case  (i.e., ambiguous case) of data distributions where success and failure of UDA are equally likely. 
This happens due to the lack of label information from the target domain. 
This case opens doors for adversarial attacks against UDA methods since a small amount of misinformation about the target domain labels can lead the UDA methods into producing a representation similar to the unfavorable case, incurring a significant increase in the target domain error.

Motivated from this analysis of UDA methods under different data distributions, 
we evaluate the extent to which the performance of current UDA methods can suffer in presence of a small amount of adversarially crafted data. For this purpose, we propose novel data poisoning attacks, using mislabeled and clean-label points. 
We evaluate the effect of our poisoning attacks 
on popular UDA methods using benchmark datasets, where they were previously shown to be very effective.
We find that our poisoning attacks cause UDA methods to either
align incorrect classes from the two domains or prevent correct classes from being very close in the representation space. 
Both of these lead to the failure of UDA methods at reducing target domain error. 
With just 10\% poison data in the source domain, target domain accuracy for current UDA methods is significantly reduced, dropping to almost 0\% in some cases.
This dramatic failure of UDA methods demonstrates their limits 
and suggests that the future UDA methods must be evaluated in adversarial settings along with evaluation on benchmark datasets to truly gauge their effectiveness at learning under the UDA setting.

Our main contributions are summarized as follows:
\begin{itemize}[leftmargin=0.75cm]
    \item We prove a lower bound on the target domain error that provides a necessary condition for successful learning in the UDA setting. Our bound shows the failure of learning a domain invariant representation while minimizing the source domain error at guaranteeing target generalization.
    \item We present example data distributions where UDA succeeds, fails, and where success and failure are equally likely. This sensitivity of UDA methods to the data distribution brings to light a new vulnerability of UDA methods to adversarial attacks such as data poisoning. 
    \item To concretely understand the extent of this vulnerability of UDA methods, we propose novel data poisoning attacks using clean-label and mislabeled data. Our results show a dramatic failure of current UDA methods at target generalization in presence of poisoned data. Thus, our poisoning attacks can provide better insights into the robustness of UDA methods than those obtained from performance evaluation on benchmark datasets.
\end{itemize}

\section{Background and related work}

{\bf Analysis of unsupervised domain adaptation:}
Several previous works have studied the problem of UDA and have provided conditions under which UDA is possible \cite{ben2007analysis,ben2010theory,mansour2009domain,mansour2014robust,david2010impossibility,ben2012hardness,wu2020representation}. 
\cite{ben2010theory, ben2007analysis} proposed an upper bound on the target domain error which has inspired many UDA algorithms. 
Recent works \cite{zhao2019learning,combes2020domain,johansson2019support} have improved the upper bounds on target domain error and have proposed a lower bound dependent on the labeling functions in the input space. Using this lower bound, the failure of learning in the UDA setting was explained in the case when marginal label distributions are different between the two domains. 
Our lower bound, on the other hand, is dependent on the labeling functions for the source and target domains, induced by the representation. 
This suggests that if a representation induces labeling functions that disagree on source and target domains then UDA provably fails even if the representation is domain invariant and minimizes error on the source domain. 
Thus our lower bound directly explains the observations of failure of UDA in many previous works \cite{ganin2016domain,liu2019transferable,wang2019characterizing,zhao2019learning,wu2020representation}. 
A detailed comparison of our work with other works analyzing the failure of learning in the UDA setting is present in Appendix~\ref{app:additional_related_work}. 

{\bf Algorithms for UDA: }
Many algorithms \cite{ganin2016domain, tzeng2017adversarial, long2017conditional, zhao2018adversarial,shu2018dirt,hoffman2018cycada} for UDA 
learn a domain invariant representation while minimizing error on the source domain. 
A popular adversarial approach to domain adaptation is DANN \cite{ganin2016domain} which uses a discriminator to distinguish points from source and target domains based on their representations.
Another popular method CDAN\cite{long2017conditional}, uses classifier output along with representations to identify the domains of the points. 
IW-DAN and IW-CDAN \cite{combes2020domain} were recently proposed as extensions of the original DANN and CDAN with an importance weighting scheme to minimize the mismatch between the labeling distributions of the two domains.
A different approach MCD \cite{saito2018maximum}, makes use of two task-specific classifiers as discriminators to align the two domains. This method adversarially trains the representation to minimize the disagreement between the two classifiers on the target domain data (classifier discrepancy) while training the classifiers to maximize this discrepancy. 
Another recent approach, SSL \cite{xu2019self} uses self-supervised learning tasks (e.g. rotation angle prediction) to better align the two domains.
In this work, we study the effect of poisoning on these methods as they have been shown to be effective at various UDA tasks.

{\bf Data poisoning:}
Data poisoning \cite{biggio2012poisoning, mei2015using, jagielski2018manipulating, chen2017targeted, ji2017backdoor, turner2018clean, mehra2020robust} is a training time attack where the attacker has access to the data which will be used by the victim for training. 
Most works \cite{zhu2019transferable, munoz2017towards, shafahi2018poison, mehra2019penalty, huang2020metapoison} have considered data poisoning in a fully supervised setting where train and test sets are drawn from the same underlying data distribution (single domain setting).
These works either target the classification of a single test point by adding a large number of poisoned data or require modifying all points in a class to affect the model's performance on that class after retraining. 
Our work studies the effect of poisoning on the entire target domain in the UDA setting using popular UDA algorithms for training. The success of our poisoning attacks at significantly reducing the target domain accuracy with a small amount of poisoned data shows the ease of poisoning in the UDA setting. In contrast, poisoning leads to a small decrease in the overall test accuracy in the single domain setting, especially when using state-of-the-art classifiers such as deep neural networks. 

\section{When does learning fail in the unsupervised domain adaption setting?}\label{sec:analysis}
{\bf Notations and settings: }
Let $\mathcal{X}$ be the data domain and $\mathcal{D}$ be the distribution over $\mathcal{X}$ with the corresponding pdf $p(x)$. We assume there is a deterministic labeling function $f:\mathcal{X}\rightarrow{}[0,1]$ for the given binary classification task. The $f(x)$ can be interpreted as $Pr[y=1|x]$.
Let $g:\mathcal{X}\rightarrow{}\mathcal{Z}$ denote the representation map that maps an input instance $x$ to its features where $\mathcal{Z}$ is called feature or representation space.
Let $h:\mathcal{Z}\rightarrow{}[0,1]$ be a hypothesis for binary classification on the representation  space. 
Note that the representation map $g$ induces a distribution over $\mathcal{Z}$ denoted by $Pr_{\tilde{\mathcal{D}}}[B] := Pr_{\mathcal{D}} [g^{-1}(B)]$ and the corresponding density function $\tilde{p}(z)$ on $\mathcal{Z}$ \cite{ben2007analysis}.
The $g$ also induces the labeling function 
\begin{equation}\label{eq:labeling function}
\tilde{f}(z) := E_\mathcal{D}[f(x) | g(x)=z],
\end{equation}
for any $B$ such that $g^{-1}(B)$ is $\mathcal{D}$-measurable. 
The misclassification error $e(h)$ w.r.t. the induced labeling function is
$e(h) = E_{z \sim \tilde{\mathcal{D}}}[| \tilde{f}(z) - h(z)|]$ where $\tilde{\mathcal{D}}$ is the induced distribution over $\mathcal{Z}$. 
Similarly, we define $e(\tilde{f},\tilde{f}') = E_{z \sim \tilde{\mathcal{D}}}[| \tilde{f}(z) - \tilde{f}'(z)|]$
and $e(h,h') = E_{z \sim \tilde{\mathcal{D}}} [| h(z) - h'(z)|]$.
The distributions $\tilde{p}$ and the labeling functions $\tilde{f}$ for the source and the target domains will be written as $\tilde{p}_S$, $\tilde{p}_T$, $\tilde{f}_S$ and $\tilde{f}_T$, respectively.
The total variation distance is $D_1(\tilde{p},\tilde{p}')=\int_{\mathcal{Z}} |\tilde{p}(z)-\tilde{p}'(z)|dz$.

\subsection{Lower bound on the target domain error}
Most adversarial DA methods learn a domain invariant representation $g:\mathcal{X}\to\mathcal{Z}$ by minimizing error on the source domain and penalizing the mismatch between the {\bf marginal} source and target distributions since the conditional distribution $\tilde{p}_T(z|y)$ for target domain is unavailable in the UDA setting.
Some works \cite{wang2019characterizing, zhao2019learning, liu2019transferable} have shown this to be insufficient at guaranteeing target generalization. Recent works have proposed a new upper bound \cite{zhao2019learning} or argued failure in terms of the upper bound 
in \cite{ben2010theory,ben2007analysis} 
({\(e_T(h) \leq \min\{e_T(\tilde{f}_S,\tilde{f}_T),e_S(\tilde{f}_S,\tilde{f}_T)\} + e_S(h) + D_1(\tilde{p}_S,\tilde{p}_T)\)}) 
being large.
%
However, 
a large upper bound does not guarantee failure. 
Thus, we prove a simple lower bound on the target domain error which shows the necessary condition for the success of learning in the UDA setting and also explains the failure of current UDA methods at guaranteeing target generalization.
\begin{theorem}\label{thm:lower_bound}
Let $\mathcal{H}$ be the hypothesis class and $\mathcal{G}$ be the class  representation maps. Then, for all $h\in \mathcal{H}$ and $g \in \mathcal{G}$, 
\begin{equation}\label{eq:lower bound}
e_T(h) \geq \max\{e_S(\tilde{f}_S,\tilde{f}_T),e_T(\tilde{f}_S,\tilde{f}_T)\}  - e_S(h) - D_1(\tilde{p}_S,\tilde{p}_T).
\end{equation}
\end{theorem}
The proof is in Appendix~\ref{app:analysis}. 
Even though our bound depends on the total variation distance which is a strict measure of distance and is difficult to estimate, the bound still explains the limitations of UDA. Extending the bound to different divergence metrics is left as future work.
The following corollary is immediate from Theorem~\ref{thm:lower_bound}.
\begin{corollary}
For all $h\in \mathcal{H}$, $g \in \mathcal{G}$, 
\[
e_T(h) \geq \max \{e_S(\tilde{f}_S,\tilde{f}_T), e_T(\tilde{f}_S,\tilde{f}_T)\} - e_S(h) - \sqrt{0.5 D_{KL}(\tilde{p}_S||\tilde{p}_T)}.
\]
\end{corollary}
This is due to the Pinsker's inequality: {\small $D_1(p,p') \leq \sqrt{0.5 D_{KL}(p||p')}$}.
The interpretation of Eq.~\ref{eq:lower bound} is as follows. 
Since $f_T$ is not observable, the goal of UDA is to minimize the observable source classification error
$e_S(h)$ and the domain mismatch $D_1(\tilde{p}_S,\tilde{p}_T)$ (or other metrics such as $d_{\mathcal{H}\Delta \mathcal{H}}(\tilde{p}_s,\tilde{p}_T)$):
\begin{equation}\label{eq:UDA}
\min_{g,h}\;e_s(h) + D_1(\tilde{p}_S,\tilde{p}_T).
\end{equation}
This leads to maximization of the second and the third term in the RHS of Eq.~\ref{eq:lower bound}.
With overparameterized models and large datasets, minimizing the empirical estimates of the quantities in Eq.~\ref{eq:UDA} can drive $e_S(h) \simeq 0$ and $D_1(\tilde{p}_S,\tilde{p}_T) \simeq 0$. Consequently, $e_T(h) \geq \max \{e_S(\tilde{f}_S,\tilde{f}_T), e_T(\tilde{f}_S,\tilde{f}_T)\}$. If $\tilde{f}_S$ and $\tilde{f}_T$ disagree on the source and target domains in the representation space, 
target domain error $e_T(h)$ will be {\bf provably} large. 
\begin{corollary}
\label{cor:two}
For all $h\in \mathcal{H}$ and $g \in \mathcal{G}$,
\begin{equation*}
|e_T(h) - {e_S(\tilde{f}_S,\tilde{f}_T)| \leq e_S(h) + D_1(\tilde{p}_S,\tilde{p}_T)},\;\;\mathrm{and}\;\;
|e_T(h) - {e_T(\tilde{f}_S,\tilde{f}_T)| \leq e_S(h) + D_1(\tilde{p}_S,\tilde{p}_T)}.
\end{equation*}
\end{corollary}
This is obtained by combining the upper and the lower bounds (Appendix~\ref{app:analysis}). Thus a UDA method can only guarantee the target error $e_T(h)$ to be close to the labeling function mismatch $e(\tilde{f}_S,\tilde{f}_T)$. Whether $e(\tilde{f}_S,\tilde{f}_T)$ becomes larger or smaller after solving Eq.~\ref{eq:UDA} is data/model dependent. For concreteness, we analyze the sensitivity of the performance of UDA methods to different data distributions using the following illustrative examples.

\begin{figure}[tb]
  \centering
    {\includegraphics[width=1\columnwidth]{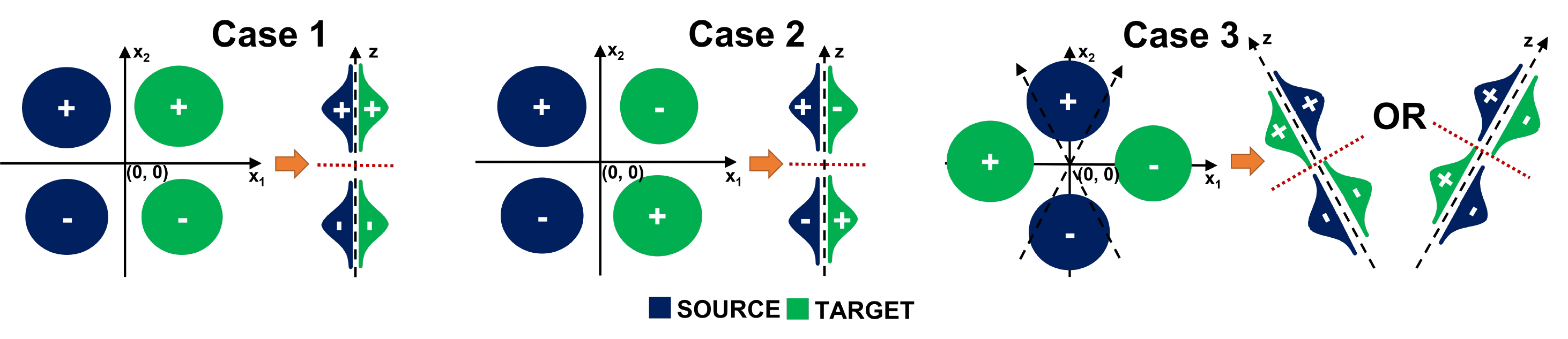}}  
  \caption{Illustrative cases for UDA. In UDA, one can observe the source data (blue blobs), the source labels (class + and -), and the target data (green blobs) but not the target labels. The optimal decision boundary (red dotted line) and linear representation $g(x)=u^Tx$ from the input space in $\mathbb{R}^2$ to the feature space in $\mathbb{R}$ (dashed line) minimizing the source and the alignment losses in Eq.~\ref{eq:UDA} can be computed accurately. Depending on the data distribution, domain adaptation can be successful (Case 1), fail (Case 2), or be undetermined (Case 3). Case 3 has two global minima with drastically different target domain performance, even though representations in both cases minimize the error on the source distribution as well as align the marginal distributions of the two domains. 
  } \vspace{-3mm}
  \label{fig:different_cases_for_uda}
\end{figure}
\subsection{Illustrative examples showing the sensitivity of UDA methods to data distributions}\label{sec:uda_cases}
Assume mixture-of-Gaussian distributions $p_S(x)$ and $p_T(x)$ for the source and the target domains in the input space, shown as blue and green blobs in Fig.~\ref{fig:different_cases_for_uda}.
We consider a linear representation map $g(x)=u^Tx$ from the input space $\mathcal{X} \subset \mathbb{R}^2$ to the feature space $\mathcal{Z} \subset \mathbb{R}$ (dashed line) that minimizes the source classification loss plus the marginal mismatch loss in Eq.~\ref{eq:UDA}.
The optimal solution $g$ to the minimization problem can be found accurately using a mix of analytical and numerical optimizations (detailed in Appendix~\ref{app:UDA_cases}).
We demonstrate the performance of UDA on three different data distributions.
In Case 1 (favorable case), the true labeling function for source and target is $f_S((x_1,x_2))=f_T((x_1,x_2))=I[x_2\geq 0]$, that is, $f$ is 1 in the upper halfspace and 0 in the bottom halfspace.
One can verify (Appendix~\ref{app:UDA_cases}) that the optimal $u$ is the vertical direction ($u=[0,1]^T$) and the best hypothesis is $h(z)=I[z\geq 0]$, in which case $e_S(h)=0$ and $D_1(\tilde{p}_S,\tilde{p}_T)=0$.
That is, perfect source classification and a perfect alignment of the marginals are achieved.
Furthermore, the true labeling function in $\mathcal{Z}$ is $\tilde{f}_S(z) = \tilde{f}_T(z) = I[z\geq0]$ (from Eq.~\ref{eq:labeling function}) and therefore $e(\tilde{f}_S,\tilde{f}_T)=0$ as well as the target loss $e_T(h)=0$.
In other words, the representation that minimizes Eq.~\ref{eq:UDA} simultaneously {\bf minimizes} $e(\tilde{f}_S,\tilde{f}_T)$, achieving the goal of reducing $e_T(h)$.
However in Case 2 (unfavorable case), the true labeling function for target is upside-down $f_T((x_1,x_2))=I[x_2\leq 0]$. Since UDA does not use the target label nor the true labeling function, the optimal $g$ is exactly the same as Case 1 ($u=[0,1]^T$), in which case we still have $e_S(h)=0$ and $D_1(\tilde{p}_S,\tilde{p}_T)=0$, but the labeling function mismatch becomes $e(\tilde{f}_S,\tilde{f}_T)=1$ as well as $e_T(h)=1$ which is the worst case.
In other words, the representation that minimizes Eq.~\ref{eq:UDA} simultaneously {\bf maximizes} $e(\tilde{f}_S,\tilde{f}_T)$, totally failing at the goal of reducing $e_T(h)$.
Case 3 exemplifies an ambiguous case. The optimal projection minimizing $e_S$ is still the vertical direction $u=[0,1]^T$ but the optimal projection minimizing the alignment loss $D_1(\tilde{p}_S,\tilde{p}_T)$ can be either of the $\pm 45^\circ$ directions $(u=[\pm 1/\sqrt{2}, 1/\sqrt{2}]^T)$ with no preference of one over the other. Therefore the optimal solution $u$ for Eq.~\ref{eq:UDA} that trades off the source error and the mismatch loss has two equal-valued global solutions $[\pm u_1,u_2]^T$ for some $u_1,u_2$. One solution (Fig.~\ref{fig:different_cases_for_uda}, Case 3, left) yields a small $e(\tilde{f}_S,\tilde{f}_T)$ and the other (Fig.~\ref{fig:different_cases_for_uda}, Case 3, right) yields a large $e(\tilde{f}_S,\tilde{f}_T)$. As can be intuitively seen, UDA is successful in the former but fails in the latter, and which equal-valued solution will be chosen is undetermined. A similar example of the failure of UDA methods due to the presence of two global optima was presented in \cite{johansson2019support}. Unlike their example, we use mixture-of-Gaussian distributions and  
analytically compute the representation that will be obtained from the minimization of Eq.~\ref{eq:UDA}. Using the obtained representations, we evaluate the quantities appearing in our lower bound and show that it is predictive of the performance of UDA methods on different data distributions. 
Details of the analysis and the results are in Appendix~\ref{app:UDA_cases}.

Motivated from this extreme sensitivity of the performance of UDA methods on simple data distributions, we set out to explore the effects of small changes in real-world data distributions (via data poisoning) on state-of-the-art UDA methods.
The results of our poisoning attacks in the next section suggest that slight changes in the data distributions could have a drastic impact on the target domain generalization performance of state-of-the-art UDA methods. Our findings highlight the importance of using adversarial settings (such as evaluating the performance of UDA methods in presence of poisoned data) to gauge the effectiveness of future UDA methods at learning in the UDA setting.
\section{Breaking unsupervised domain adaptation methods with data poisoning}\label{sec:experiments}
In this section, we present our novel data poisoning attacks to evaluate the ease with which current UDA methods can be fooled into producing a representation that leads to a large error on the target domain\footnote{Our code can be found at \url{https://github.com/akshaymehra24/LimitsOfUDA}.}.
We propose three methods to generate poisoned data which will be added to the clean source domain data.
The first poisoning attack uses mislabeled data as poisons.
Under this attack, we evaluate two approaches (a) adding mislabeled source domain data (wrong-label correct-domain poisoning) and (b) adding mislabeled target domain data (wrong-label incorrect-domain poisoning) as poisons. 
The second attack adds images from the source domain watermarked with images from the target domain with incorrect labels (watermarking attack) as poisons.
The last poisoning attack uses poisoned data with clean labels. We evaluate two approaches for this attack (a) using source domain data (clean-label correct-domain poisoning) and (b) using target domain data (clean-label wrong-domain poisoning) to initialize the poison data.  
The intuitive pictures of how these poisoning attacks hurt UDA methods are shown in Figs.~\ref{fig:explanation_of_attacks_a},~\ref{fig:explanation_of_attacks_b}, and ~\ref{fig:explanation_of_attacks_c} (details in Appendix~\ref{app:fig_explanation}).
We evaluate the effect of poisoning on popular UDA methods namely, DANN\cite{ganin2016domain}, CDAN\cite{long2017conditional}, MCD\cite{saito2018maximum}, SSL\cite{xu2019self} (with rotation-angle prediction task), IW-DAN\cite{combes2020domain}, and IW-CDAN\cite{combes2020domain}.
We compare the difference in the target accuracy attainable by these methods when using clean versus poisoned data. 
Two benchmark datasets are used in our experiments, namely Digits and Office-31. 
We evaluate four tasks using SVHN, MNIST, MNIST\_M, and USPS datasets under Digits and six tasks under the Office-31 using Amazon (A), DSLR (D), and Webcam (W) datasets.
Additional experiments evaluating the performance of our poisoning attacks against other popular UDA methods are present in Appendix~\ref{app:additional_experiments}.
For all experiments, we train UDA algorithms 
using neural networks whose architectures are similar to those used by previous works
(see Appendix~\ref{app:experiments_details} for details). 
\subsection{Poisoning using mislabeled source and target domain data}\label{sec:attack_1}
In this experiment, mislabeled data is added to the clean source domain as poison.  
The labels provided by the attacker to the poison data are chosen to fool UDA methods into learning a representation that incurs large errors on the target domain.
We use two simple and effective labeling functions for this. 
For the Digits dataset, we use a labeling function that systematically labels the poison data to the class next to their true class (e.g. poison points with true class one are labeled as two, points with true class two are labeled as three, and so on). 
For the Office-31 dataset, we use a labeling function that assigns the poison data the label of the closest (in the representation space learned using the clean source domain data) incorrect source domain class. 
Here the attacker is limited to adding only 10\% poisoned data with respect to the size of the available target domain data (for experiments with different poison percentages see Appendix~\ref{app:poison_percentage}).
In wrong-label correct-domain poisoning, mislabeled source domain data are used as poisons. 
The results in rows marked with Poison$_\mathrm{source}$ in Tables~\ref{Table:digits_experiment_1} and~\ref{Table:office31_experiment_1}, show that UDA methods suffer only a minor decrease in target domain accuracy with this approach. This happens because the presence of a large amount of correctly labeled source domain data prevents the small amount of poisoned data from affecting the performance of UDA methods.
A similar effect is observed when a small amount of mislabeled data is used for poisoning in the traditional single domain setting \cite{mehra2019penalty, munoz2017towards}. 
However, if the relative size of poisoned data is larger or comparable to the size of clean source domain data, poisoning can be effective. This is observed in Table~\ref{Table:office31_experiment_1} when Amazon is the target data. The Amazon dataset is roughly 5 times bigger than both DSLR and Webcam datasets. Due to this, the permissible amount of poisoned data (10\% of the target domain) makes the size of clean and poisoned data comparable, leading to successful poisoning. 
Thus, this attack causes UDA methods to fail in presence of a large amount of poisoned data.  
\begin{figure}[tb]
  \centering
  \subfigure[DANN trained on clean data]{\includegraphics[width=0.23\columnwidth]{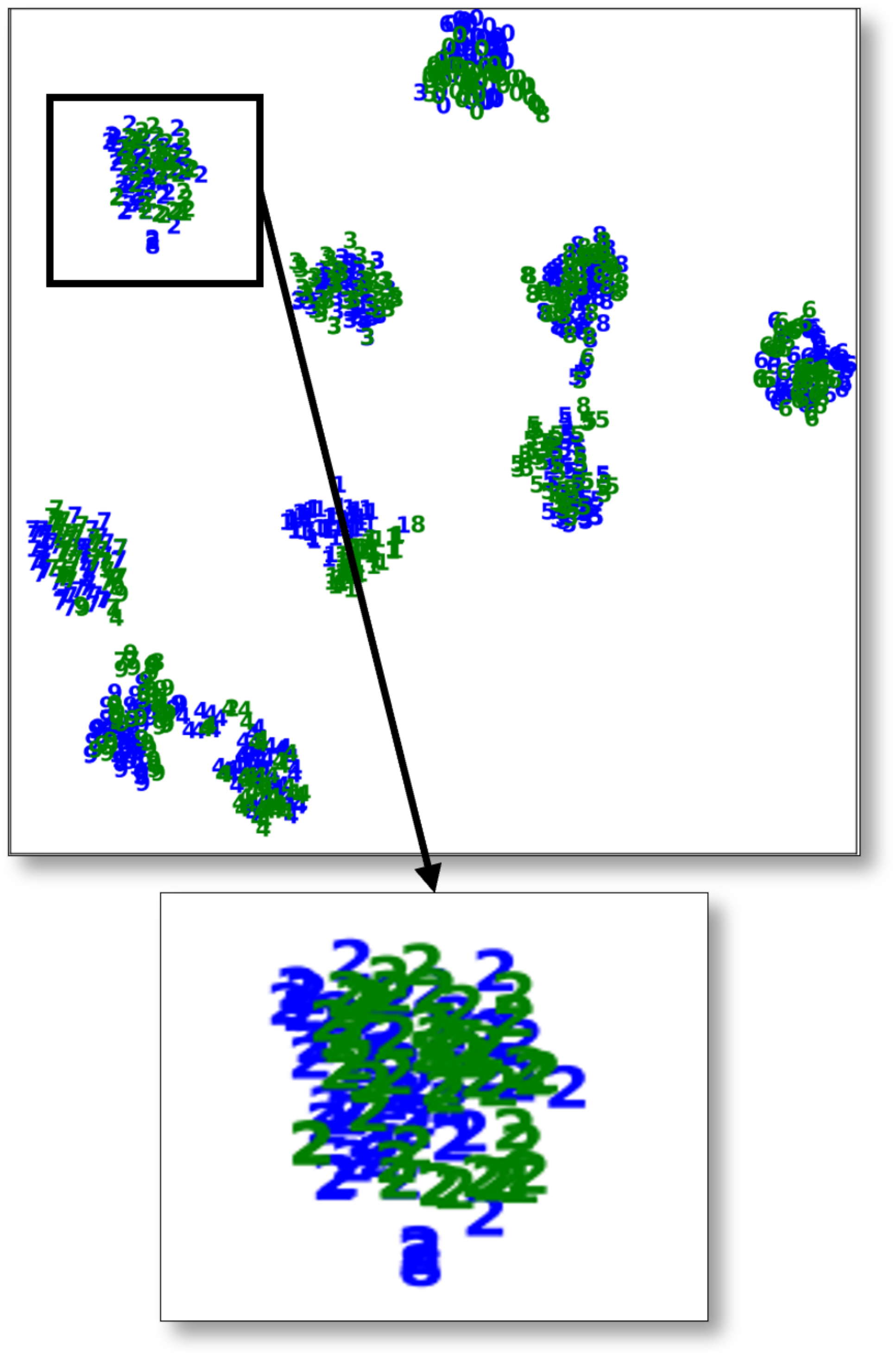}}
  \hfill
  \subfigure[DANN trained on poisoned data]{\includegraphics[width=0.23\columnwidth]{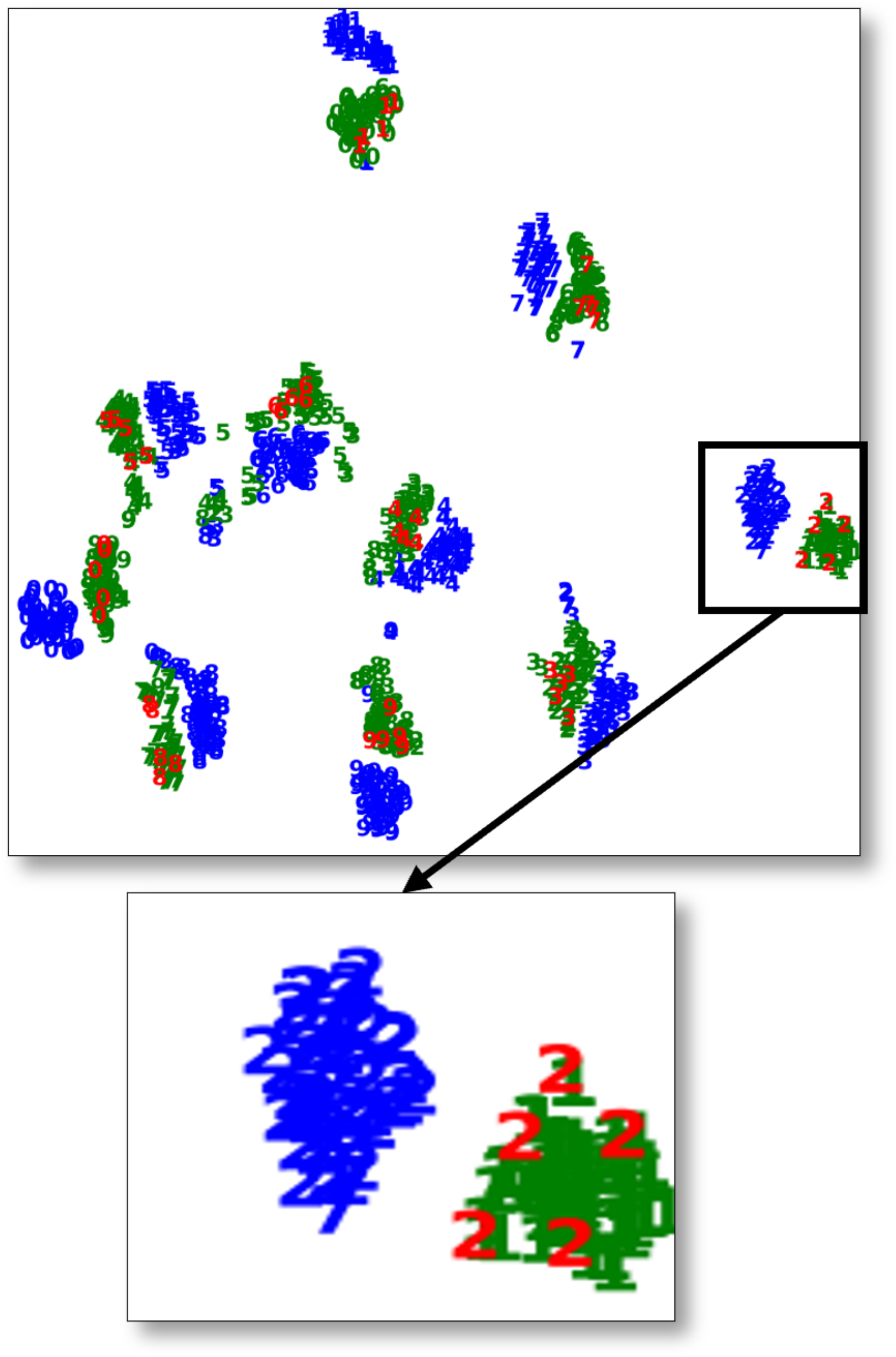}}
  \hfill
  \subfigure[CDAN trained on clean data]{\includegraphics[width=0.23\columnwidth]{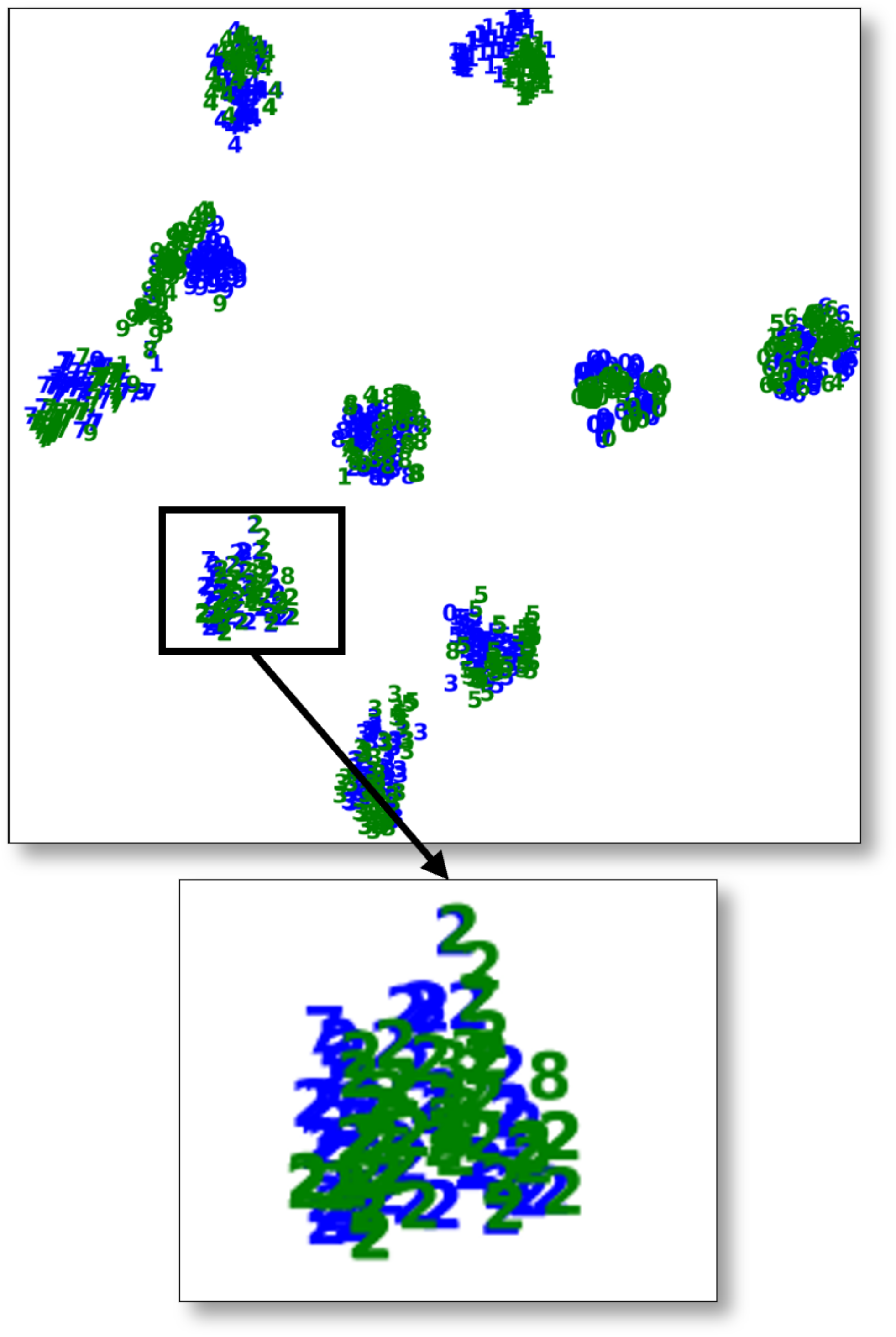}}
  \hfill
  \subfigure[CDAN trained on poisoned data]{\includegraphics[width=0.23\columnwidth]{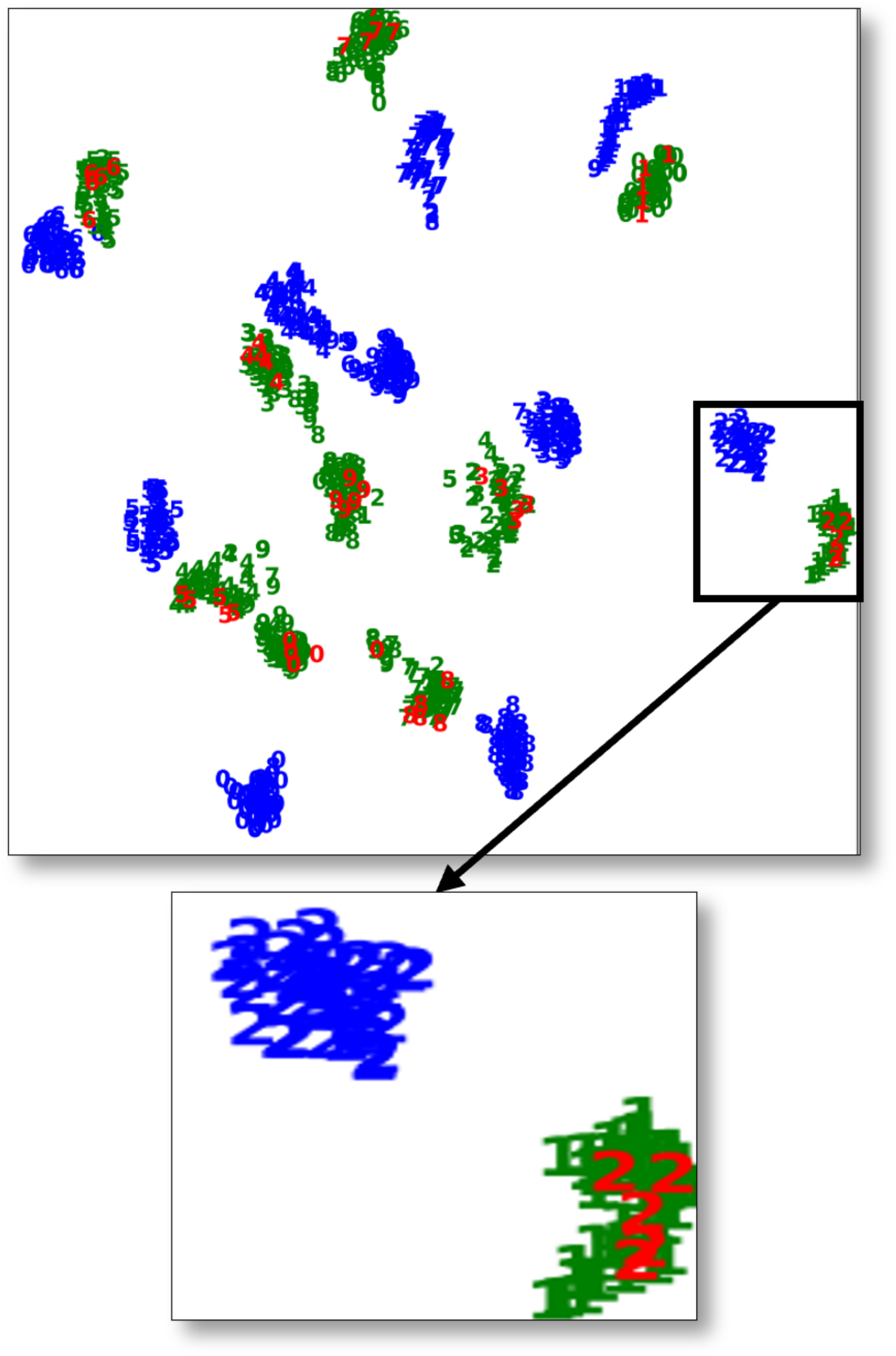}}
  \includegraphics[width=0.3\textwidth]{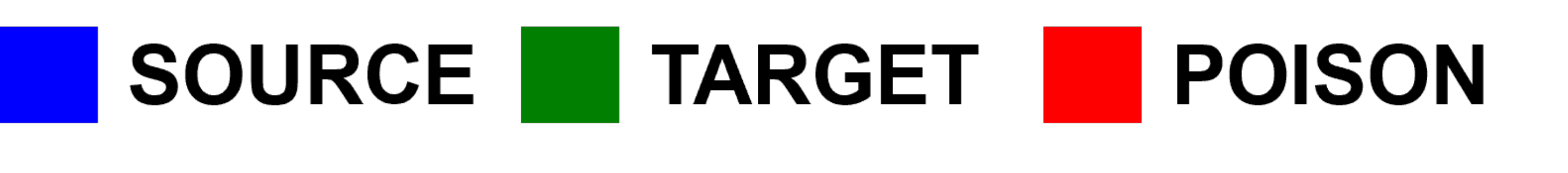}
  \caption{ 
  (Best viewed in color.) t-SNE embedding of the data in the representation space (for MNIST $\rightarrow{}$USPS task) learned using DANN and CDAN on clean and poisoned source domain data. Without poisoning, correct classes (data from source class 2 is zoomed in) from two domains are aligned ((a) and (c)). The presence of poisoned data fools the methods into aligning incorrect classes from the two domains ((b) and (d)). The mismatch between the source and target classes is dependent on the labels of the poison data (due to which the target class 1 is aligned to the source class 2).}
  \label{fig:exp_1_a}
\end{figure}
\begin{figure}[tb]
  \centering
  \subfigure[MCD trained on clean data]{\includegraphics[width=0.23\columnwidth]{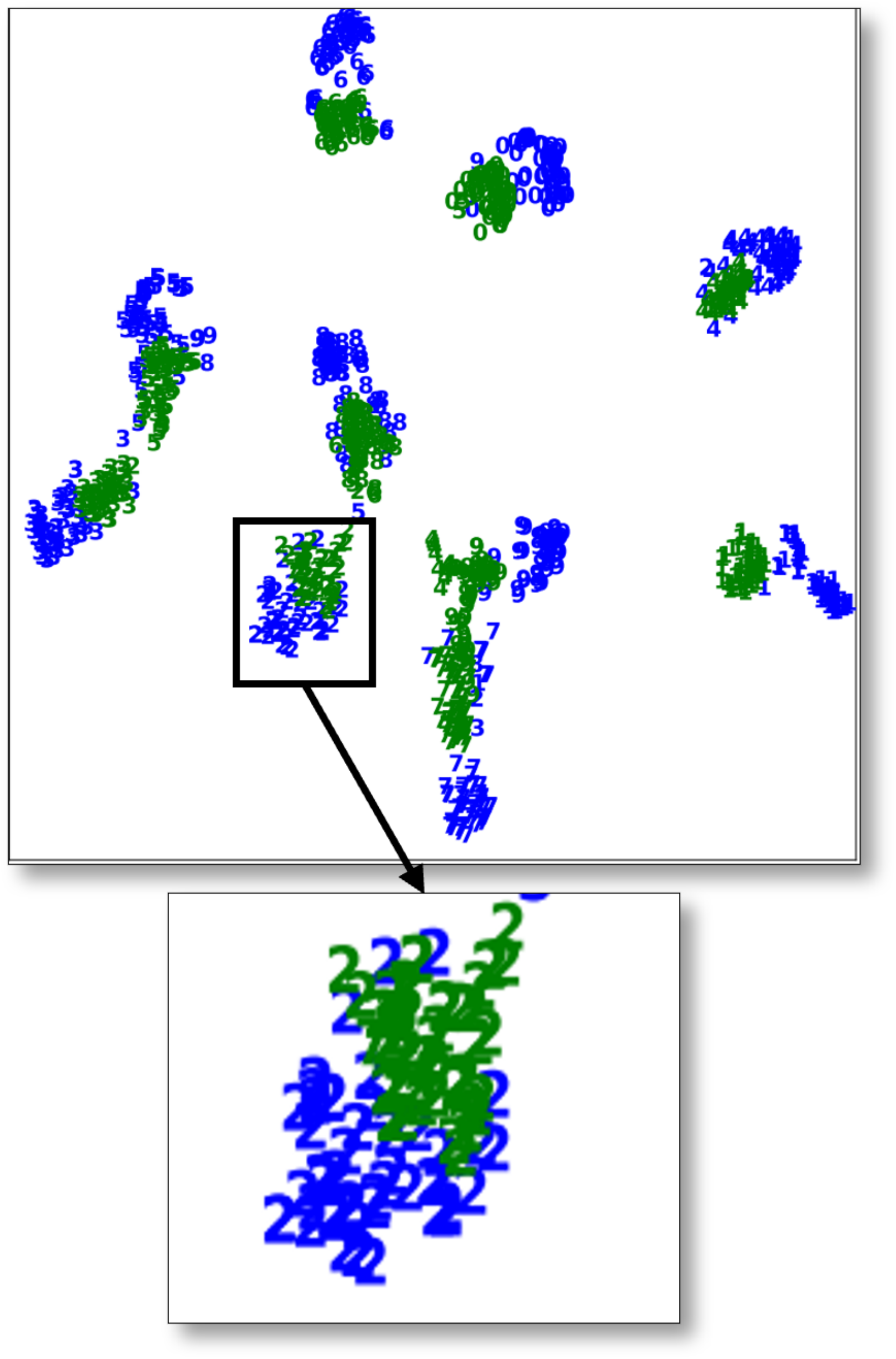}}
  \hfill
  \subfigure[MCD trained on poisoned data]{\includegraphics[width=0.23\columnwidth]{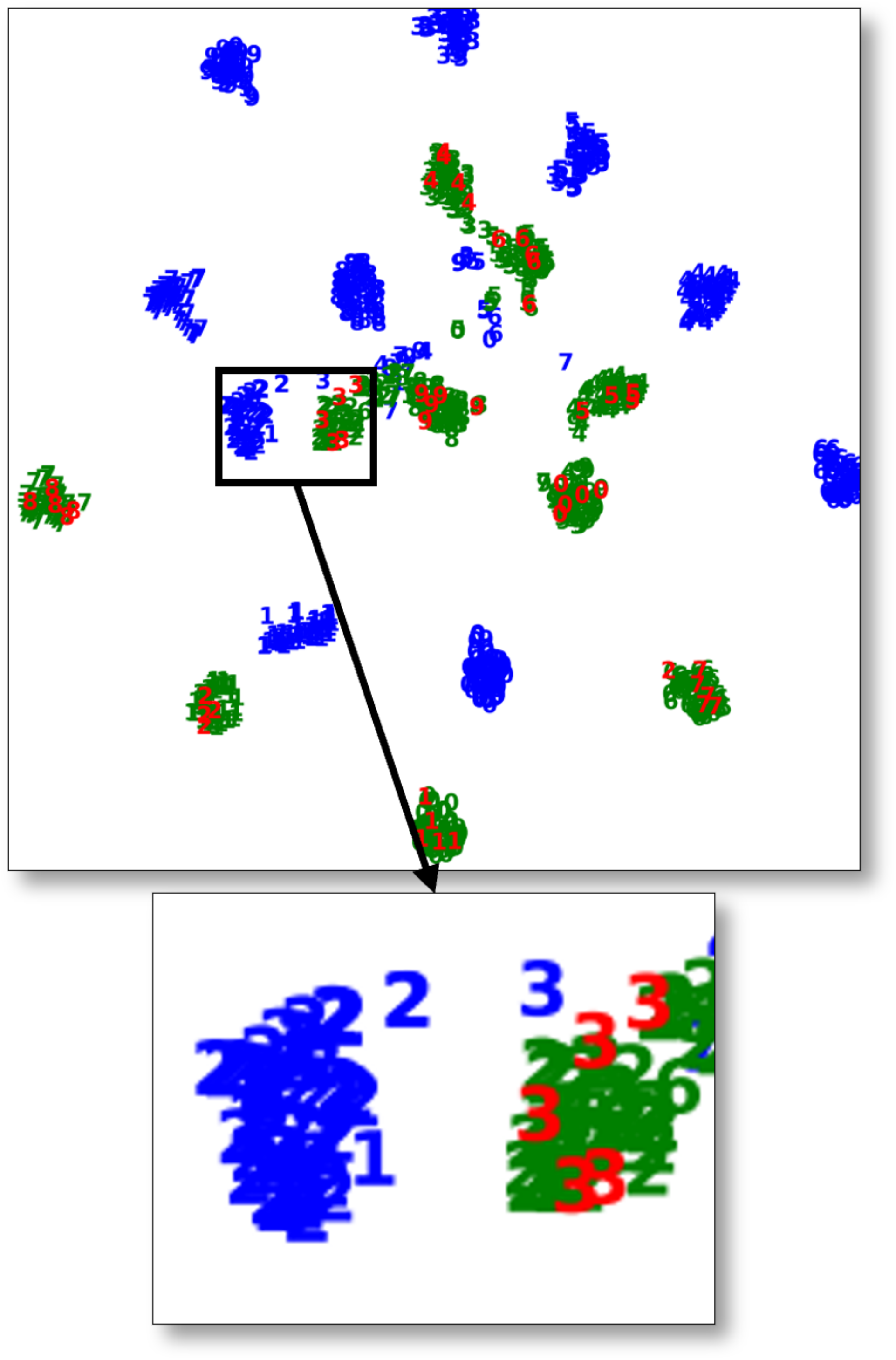}}
  \hfill
  \subfigure[SSL trained on clean data]{\includegraphics[width=0.23\columnwidth]{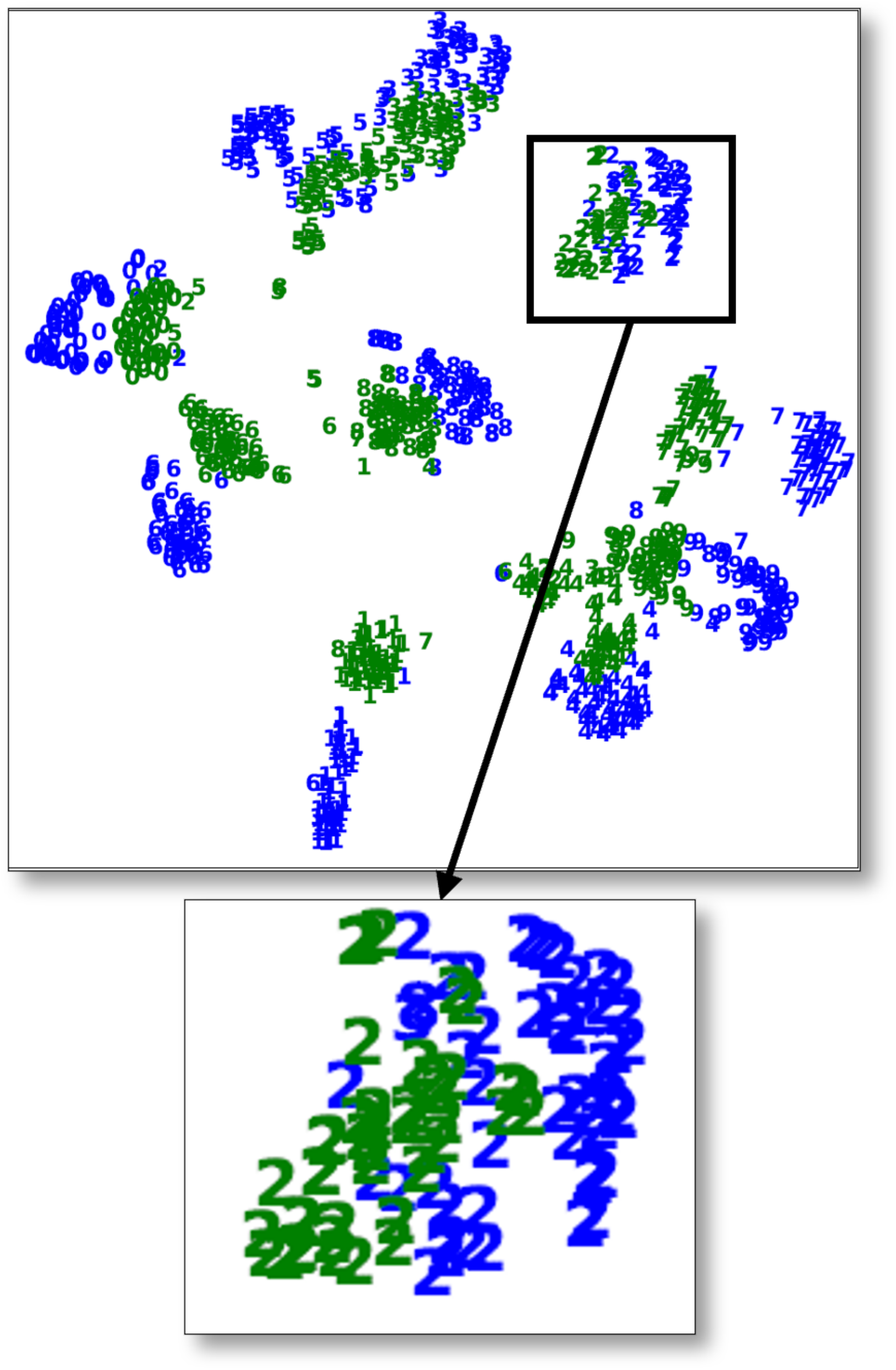}}
  \hfill
  \subfigure[SSL trained on poisoned data]{\includegraphics[width=0.23\columnwidth]{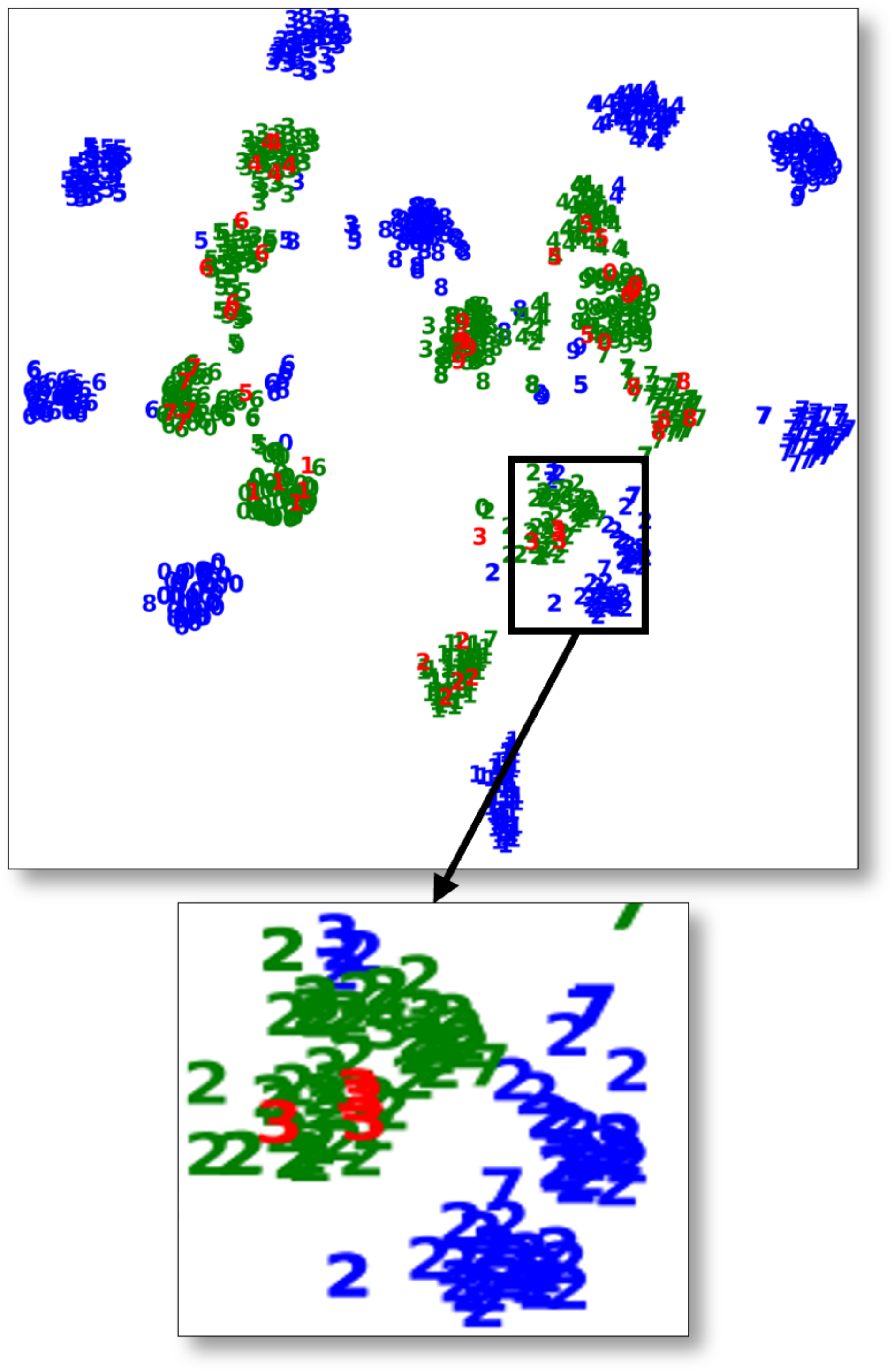}}
  \includegraphics[width=0.3\textwidth]{images/Legend_2.pdf}
  \caption{
  (Best viewed in color.) t-SNE embedding of the data in the representation space (for MNIST $\rightarrow{}$USPS task) learned using MCD and SSL on clean and poisoned source domain data. Without poisoning, correct classes (data from source class 2 is zoomed in) from two domains are aligned ((a) and (c)). The presence of poisoned data prevents the methods from aligning correct classes from the two domains ((b) and (d)).}
  \label{fig:exp_1_b}
\end{figure}
\begin{table}[tb]
  \caption{Decrease in the target domain accuracy for UDA methods trained on poisoned source domain data (with poisons sampled from source/target domains) compared to accuracy attained with clean data on the Digits tasks (mean$\pm$s.d. of 5 trials). 
  }
  \label{Table:digits_experiment_1}
  \centering
  \small
  \resizebox{0.9\columnwidth}{!}{
    \begin{tabular}{c|c|cccc}
    \toprule
    Method & Data & 
    SVHN $\rightarrow{}$ MNIST & MNIST $\rightarrow{}$  MNIST\_M & MNIST $\rightarrow{}$ USPS 
    & USPS $\rightarrow{}$ MNIST \\
    \midrule
    Source only & Clean & 72.42$\pm$1.44 & 39.05$\pm$2.30 & 87.13$\pm$1.75 & 78.6$\pm$1.45\\
    \midrule
    \multirow{3}{*}{DANN} & Clean & 78.05$\pm$1.15 & 76.22$\pm$2.38 & 92.17$\pm$0.73 & 92.73$\pm$0.71 \\
    & Poison$_\mathrm{source}$ & 70.26$\pm$2.84 & 69.98$\pm$3.49 & 93.44$\pm$0.84 & 92.08$\pm$0.68\\
    & Poison$_\mathrm{target}$ & {\bf 1.46$\pm$1.12} & {\bf0.48$\pm$0.04} & {\bf0.97$\pm$0.53} & {\bf5.83$\pm$0.82} \\
    \midrule
    \multirow{3}{*}{CDAN} & Clean & 79.19$\pm$0.70 & 73.88$\pm$1.10 & 93.92$\pm$0.97 & 95.94$\pm$0.71\\
    & Poison$_\mathrm{source}$ & 73.67$\pm$4.19 & 73.36$\pm$1.31 & 92.06$\pm$0.59 & 92.85$\pm$0.31\\
    & Poison$_\mathrm{target}$ & {\bf12.27$\pm$5.02} & {\bf0.59$\pm$0.12} & {\bf1.92$\pm$0.42} & {\bf2.96$\pm$0.71} \\
    \midrule
    \multirow{3}{*}{MCD} & Clean & 96.18$\pm$1.53 & 93.95$\pm$0.33 & 89.96$\pm$2.04 & 88.34$\pm$2.50\\
    & Poison$_\mathrm{source}$ & 85.86$\pm$5.66 & 93.33$\pm$0.71 & 87.99$\pm$1.05 & 83.19$\pm$2.98 \\
    & Poison$_\mathrm{target}$ & {\bf0.97$\pm$0.94} & {\bf0.37$\pm$0.06} & {\bf0.66$\pm$0.16} & {\bf2.07$\pm$0.69}\\
    \midrule
    \multirow{3}{*}{SSL} & Clean & 66.85$\pm$2.30 & 92.76$\pm$0.91 & 88.69$\pm$1.28 & 82.23$\pm$1.59\\
    & Poison$_\mathrm{source}$ & 61.97$\pm$1.62 & 91.35$\pm$1.13 & 85.74$\pm$2.92 & 82.56$\pm$0.84\\
    & Poison$_\mathrm{target}$ & {\bf0.31$\pm$0.03} & {\bf0.36$\pm$0.02} & {\bf7.76$\pm$1.52} & {\bf9.88$\pm$1.07}\\
    \bottomrule
    \end{tabular}
    }
\end{table}

\begin{table}[tb]
  \caption{Decrease in the target domain accuracy for UDA methods trained on poisoned source domain data (with poisons sampled from source/target domains) compared to accuracy attained with clean data on the Office tasks (mean$\pm$s.d. of 3 trials).}
  \label{Table:office31_experiment_1}
  \centering
  \small
  \resizebox{0.9\columnwidth}{!}{
    \begin{tabular}{c|c|cccccc}
    \toprule
    Method & Dataset & 
    A $\rightarrow{}$ D & A $\rightarrow{}$ W & D $\rightarrow{}$ A 
    & D $\rightarrow{}$ W & W $\rightarrow{}$ A & W $\rightarrow{}$ D\\
    \midrule
    Source Only & Clean & 79.61 & 73.18 & 59.33 & 96.31 &	58.75 &	99.68\\
    \midrule
    \multirow{3}{*}{DANN} & Clean & 84.06 &	85.41 &	64.67 &	96.08 &	66.77 &	99.44 \\
    & Poison$_\mathrm{source}$ & 79.11$\pm$0.35 & 83.98$\pm$1.19 & 44.31$\pm$2.94 & 95.22$\pm$0.22 & 43.35$\pm$1.65 & 96.58$\pm$0.87\\
    & Poison$_\mathrm{target}$ & {\bf59.83$\pm$0.20} & {\bf63.18$\pm$1.96} & {\bf17.58$\pm$0.39} & {\bf76.43$\pm$0.62} & {\bf19.82$\pm$0.33} & {\bf84.20$\pm$0.71}\\
    \midrule
    \multirow{3}{*}{CDAN} & Clean & 89.56 &	93.01 &	71.25 &	99.24 &	70.32 &	100\\
    & Poison$_\mathrm{source}$ & 90.16$\pm$0.61 & 90.94$\pm$0.13 & 53.68$\pm$0.37 & 98.45$\pm$0.07 & 57.27$\pm$0.57 & 99.66$\pm$0.23\\
    & Poison$_\mathrm{target}$ & {\bf71.88$\pm$0.20} & {\bf71.94$\pm$0.76} & {\bf11.19$\pm$1.47} & {\bf86.37$\pm$0.36} & {\bf18.54$\pm$0.45} & {\bf89.08$\pm$1.23} \\
    \midrule
    \multirow{3}{*}{IW-DAN} & Clean & 84.3 &	86.42 &	68.38 &	97.13 &	67.16 &	100\\
    & Poison$_\mathrm{source}$ & 81.25$\pm$0.91 & 83.27$\pm$0.45 & 50.76$\pm$1.58 & 96.68$\pm$0.29 & 48.31$\pm$2.02 & 99.73$\pm$0.12\\
    & Poison$_\mathrm{target}$ & {\bf61.64$\pm$0.53} & {\bf63.43$\pm$1.14} & {\bf15.69$\pm$1.76} & {\bf80.29$\pm$0.07} & {\bf26.54$\pm$0.48} & {\bf88.62$\pm$0.23}\\
    \midrule
    \multirow{3}{*}{IW-CDAN} & Clean & 88.91	& 93.23 & 71.9 & 99.3 &	70.43 &	100\\
    & Poison$_\mathrm{source}$ & 89.83$\pm$0.31 & 90.77$\pm$1.27 & 57.51$\pm$0.06 & 98.41$\pm$0.07 & 61.16$\pm$1.21 & 99.66$\pm$0.12\\
    & Poison$_\mathrm{target}$ & {\bf72.62$\pm$0.42} & {\bf70.15$\pm$2.21} & {\bf14.36$\pm$0.66} & {\bf88.26$\pm$0.15} & {\bf22.36$\pm$0.96} & {\bf87.75$\pm$0.53} \\
    \bottomrule
    \end{tabular}
    }
\end{table}
In the wrong-label wrong-domain approach, mislabeled target domain data is used for poisoning.
The effect of poisoning on discriminator-based methods \cite{ganin2016domain,long2017conditional} is shown in Fig.~\ref{fig:explanation_of_attacks_a}. 
The domain discriminator used in these methods is maximally confused when marginal distributions of the source and target domains are aligned. 
However, alignment of the marginal distributions does not ensure alignment of the conditional distributions \cite{zhao2019learning,long2017conditional}. 
The objective of achieving low source domain error pushes the source domain classifier to correctly classify the poison data. 
Thus placing the poison and source data with the same labels close in the representation space.
\begin{wrapfigure}[15]{r}{0.28\textwidth}
\vspace{-0.3cm}
\centering
\includegraphics[width=0.25\columnwidth]{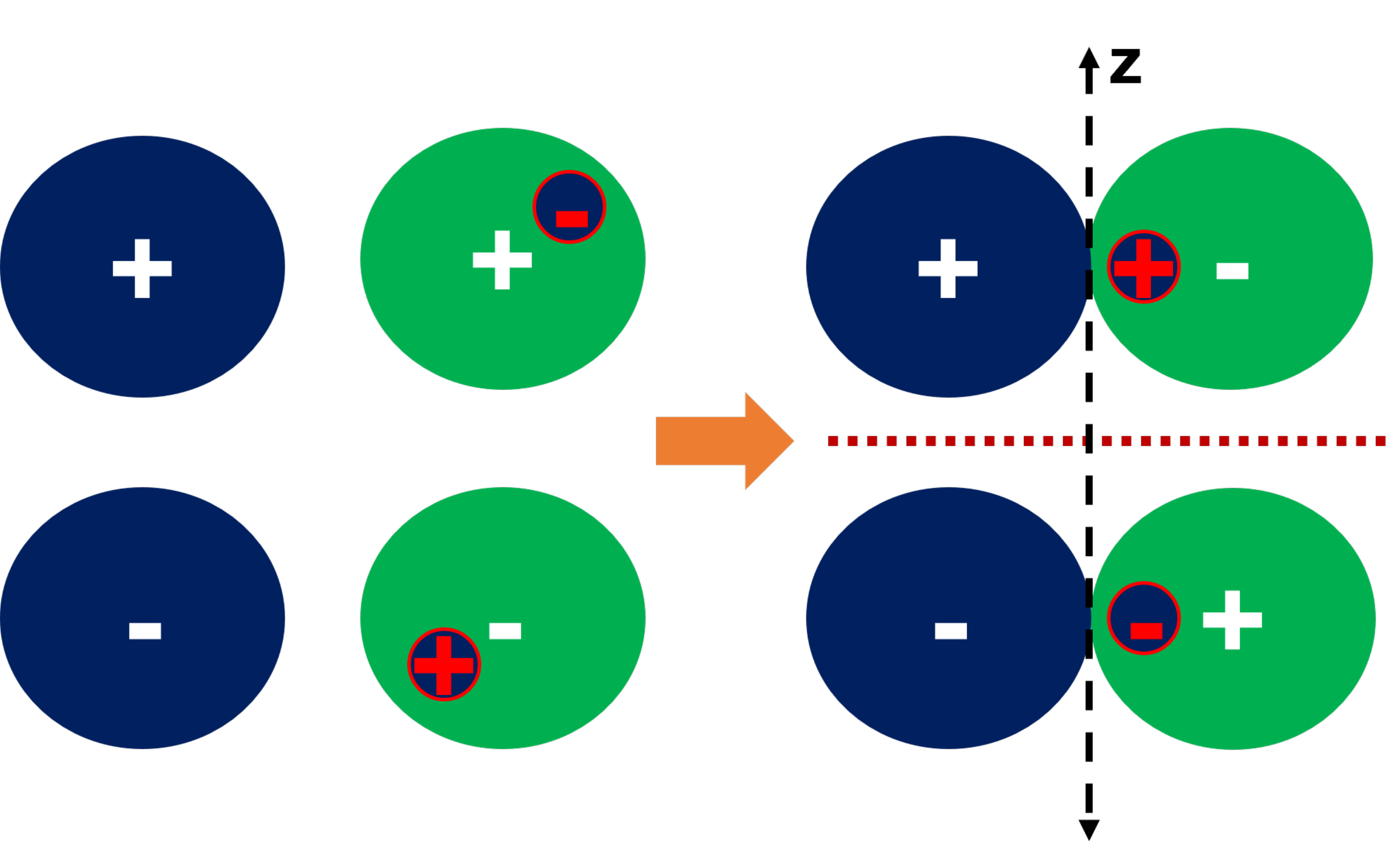}
\includegraphics[width=0.2\columnwidth]{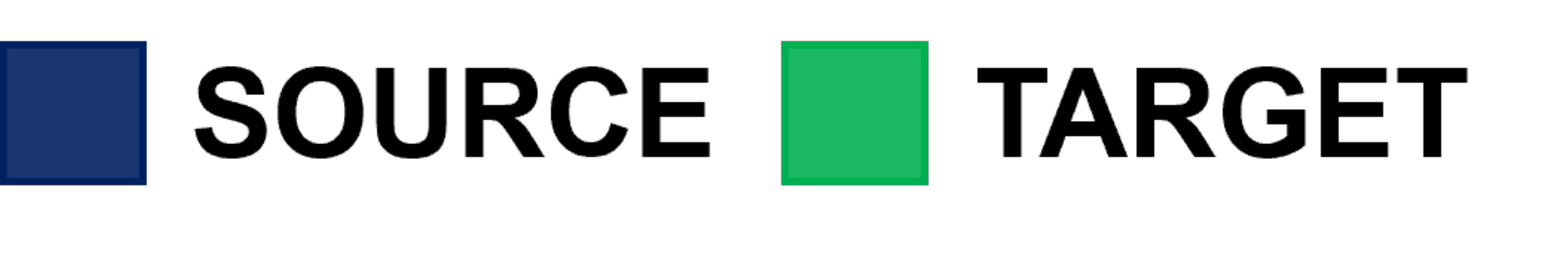}
\vspace{-0.3cm}
\caption{ 
Wrong-label incorrect-domain poisoning causes discriminator-based UDA approaches to align wrong classes (+ to -) from the two domains, leading to a significant decline in the target domain accuracy.}
\label{fig:explanation_of_attacks_a}
\end{wrapfigure}
Since the poison data is mislabeled target domain data, poisoning makes UDA methods align wrong source and target domain classes. 
This leads to a significant decline in the target domain accuracy. 

This is also evident from the t-SNE embedding for DANN and CDAN methods in Fig.~\ref{fig:exp_1_a}. 
In absence of poisoning, correct source and target classes are aligned (Fig.~\ref{fig:exp_1_a} (a) and (c)), whereas in presence of poisoning wrong classes from the two domains are closer (Fig.~\ref{fig:exp_1_a} (b) and (d)). 
For MCD \cite{saito2018maximum}, which uses use classifier discrepancy to detect and align source and target domains, our poisoned data prevents the method from detecting target examples. This happens because the term that minimizes the error on the poisoned source domain reduces the discrepancy of the classifiers on poison data, which are from the target domain.
Thus, both the generator (common representation) and discriminator (in the form of two classifiers) become optimal and there is no signal for the generator to align the two domains. 
The t-SNE embedding in Fig.~\ref{fig:exp_1_b} shows this effect. In presence of poisoned data Fig.~\ref{fig:exp_1_b} (b), we see twenty distinct clusters rather than just ten (we have ten classes in the Digits) as seen in the absence of poisoning in Fig.~\ref{fig:exp_1_b} (a).
In SSL \cite{xu2019self}, 
the generator must work well on the main task, i.e., should correctly classify all data in the poisoned source domain data. 
The auxiliary task ensures that representations of the source and target domains become similar as seen on clean data (Fig.~\ref{fig:exp_1_b} (c)). 
But in presence of poisoned data, similar representations of correct source and target domain classes leads to a drop in the accuracy of the main task on the poisoned data. 
This creates a conflict between the main and auxiliary tasks due to which correct source and target domain classes cannot be aligned (Fig.~\ref{fig:exp_1_b} (d)). The objective of making the main task accurate on poisoned data leads to target domain data being assigned the labels of the poisoned points. Thereby leading to a significant drop in the target-domain accuracy. 
The results of wrong-label wrong-domain poisoning 
present in rows marked with Poison$_\mathrm{target}$ in Tables~\ref{Table:digits_experiment_1} and~\ref{Table:office31_experiment_1} 
show a significant reduction in the target domain accuracy compared to the accuracy obtained on clean data.
On Digits, poisoning makes the target domain accuracy close to 0\% on most tasks. 
On Office-31, poisoning causes at least a 20\% reduction in the target domain accuracy in most cases. 
The reason tasks in Office-31 are less hurt by poisoning is because of the use of a pre-trained representation.   
Similar to previous works for Office-31, we use all labeled source and unlabeled target domain data for training and fine-tune from a representation pre-trained on the ImageNet dataset.
The slowly changing representation pre-trained on a massive dataset weakens the effect of poisoning but cannot eliminate it. 
For the two tasks D $\rightarrow{}$ W and W $\rightarrow{}$ D, the fine-tuned representation, trained just on the clean source dataset (Source Only in Table~\ref{Table:office31_experiment_1}) achieves high accuracy indicating the domains are already well aligned in terms of conditional distributions as well. 
As a result, UDA methods can
easily align correct classes and suffer a drop close to 10\% which is the amount of poisoned data added. However, this is not an interesting case for the evaluation of UDA methods as domains can be aligned just by training on the source domain data without the need for target domain data. 

\begin{table}
  \caption{Decrease in target accuracy when training different domain adaptation methods on poisoned watermarked data in comparison to the target accuracy obtained with clean data on the Digits task (mean$\pm$s.d. of 5 trials).}
  \label{Table:digits_experiment_2}
  \centering
  \small
  \resizebox{0.9\columnwidth}{!}{
    \begin{tabular}{c|c|cccc}
    \toprule
    Method & Dataset & SVHN $\rightarrow{}$ MNIST & MNIST $\rightarrow{}$  MNIST\_M & MNIST $\rightarrow{}$ USPS & USPS $\rightarrow{}$ MNIST \\
    \midrule
    \multirow{4.5}{*}{DANN} & Clean & 78.05$\pm$1.15 & 76.22$\pm$2.38 & 92.17$\pm$0.73 & 92.73$\pm$0.71 \\
    \cmidrule{2-6}
    & \multirow{3}{*}{Poisoned$_\alpha$} & 68.76$\pm$3.91$_{0.05}$ & 27.36$\pm$15.77$_{0.05}$ & 91.84$\pm$0.55$_{0.10}$ & 88.93$\pm$4.36$_{0.10}$\\
    & & 57.96$\pm$5.84$_{0.10}$ & 7.19$\pm$2.59$_{0.10}$ & 85.51$\pm$3.01$_{0.20}$ & 78.29$\pm$8.52$_{0.20}$ \\
    & & 33.33$\pm$4.38$_{0.15}$ & 4.73$\pm$0.38$_{0.15}$ & 39.29$\pm$1.34$_{0.30}$ & 41.52$\pm$7.43$_{0.30}$ \\
    
    \midrule
    \multirow{4.5}{*}{CDAN} & Clean & 79.19$\pm$0.70 & 73.88$\pm$1.10 & 93.92$\pm$0.97 & 95.94$\pm$0.71\\
    \cmidrule{2-6}
    & \multirow{3}{*}{Poisoned$_\alpha$} & 65.77$\pm$4.82$_{0.05}$ & 55.47$\pm$3.87$_{0.05}$ & 92.05$\pm$0.96$_{0.10}$ & 86.53$\pm$1.55$_{0.10}$\\
    & & 57.57$\pm$3.11$_{0.10}$ & 7.37$\pm$1.26$_{0.10}$ & 86.54$\pm$2.43$_{0.20}$ & 77.39$\pm$4.84$_{0.20}$ \\
    & & 44.83$\pm$4.09$_{0.15}$ & 6.68$\pm$1.64$_{0.15}$ & 88.67$\pm$0.44$_{0.30}$ & 79.54$\pm$7.02$_{0.30}$ \\

    \midrule
    \multirow{4.5}{*}{MCD} & Clean & 96.18$\pm$1.53 & 93.95$\pm$0.33 & 89.96$\pm$2.04 & 88.34$\pm$2.50\\
    \cmidrule{2-6}
    & \multirow{3}{*}{Poisoned$_\alpha$} & 74.96$\pm$3.20$_{0.05}$ & 92.18$\pm$0.78$_{0.05}$ & 6.75$\pm$4.81$_{0.10}$ & 30.35$\pm$2.30$_{0.10}$\\
    & & 35.85$\pm$3.23$_{0.10}$ & 85.38$\pm$3.57$_{0.10}$ & 0.77$\pm$0.22$_{0.20}$ & 11.34$\pm$0.77$_{0.20}$ \\
    & & 17.01$\pm$1.52$_{0.15}$ & 70.34$\pm$11.49$_{0.15}$ & 0.71$\pm$0.22$_{0.30}$ & 3.28$\pm$0.94$_{0.30}$ \\

    \midrule
    \multirow{4.5}{*}{SSL} & Clean & 66.85$\pm$2.30 & 92.76$\pm$0.91 & 88.69$\pm$1.28 & 82.23$\pm$1.59\\
    \cmidrule{2-6}
    & \multirow{3}{*}{Poisoned$_\alpha$} & 44.64$\pm$2.01$_{0.05}$ & 53.33$\pm$13.48$_{0.05}$ & 32.38$\pm$10.77$_{0.10}$ & 34.72$\pm$1.71$_{0.10}$\\
    & & 10.86$\pm$1.21$_{0.10}$ & 26.64$\pm$10.1$_{0.10}$ & 6.12$\pm$2.13$_{0.20}$ & 21.86$\pm$1.01$_{0.20}$ \\
    & & 3.4$\pm$1.11$_{0.15}$ & 12.14$\pm$4.66$_{0.15}$ & 2.42$\pm$0.41$_{0.30}$ & 11.90$\pm$0.81$_{0.30}$ \\

    \bottomrule
    \end{tabular}
    }
\end{table}

\subsection{Poisoning with mislabeled watermarked data}\label{sec:attack_2}
\begin{wrapfigure}[14]{r}{0.27\textwidth}
\vspace{-0.8cm}
\centering
\includegraphics[width=0.25\columnwidth]{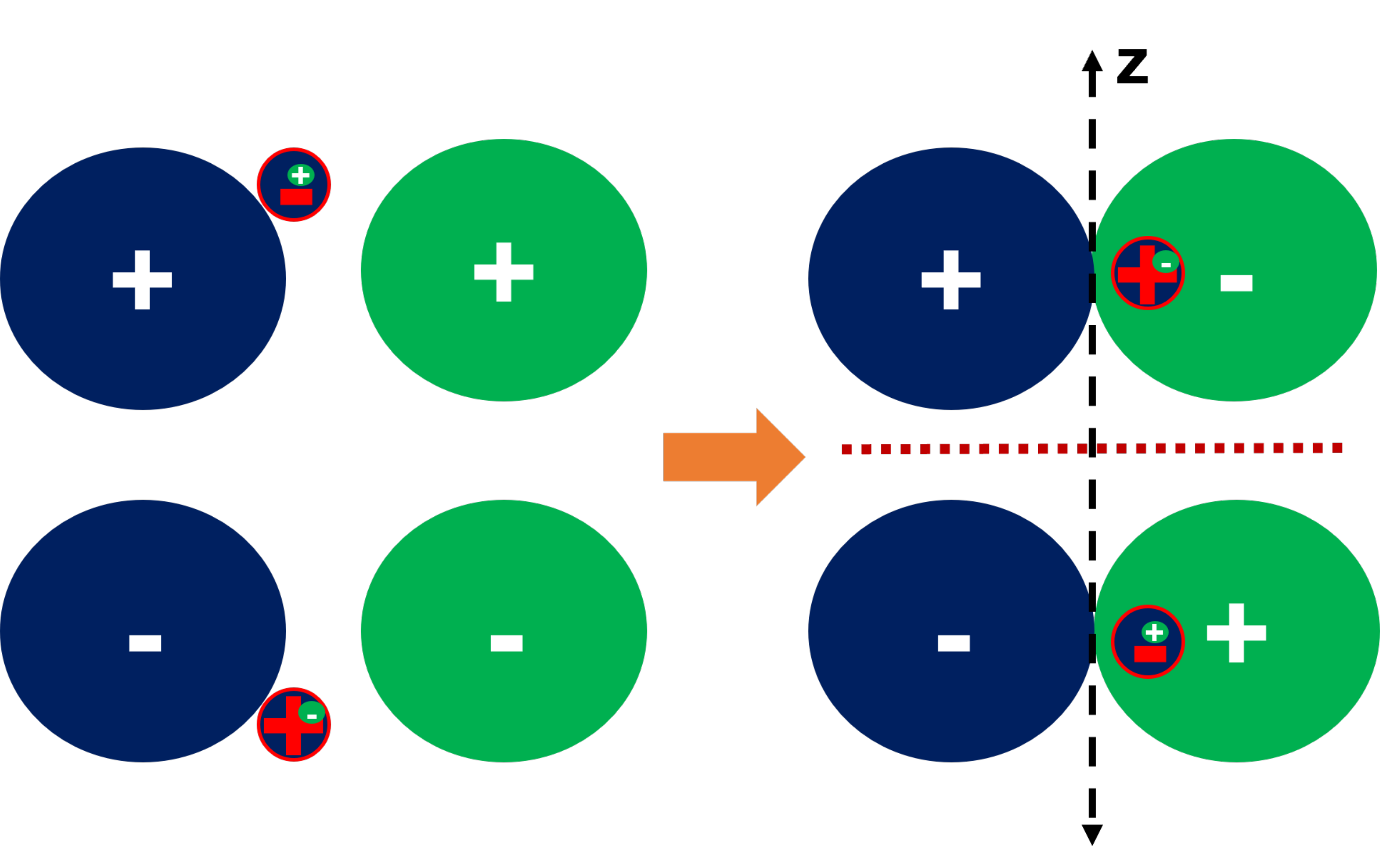}
\includegraphics[width=0.2\columnwidth]{images/Legend_0.pdf}
\vspace{-0.1cm}
\caption{Successful poisoning with mislabeled watermarked data prevents discriminator-based UDA approaches from aligning correct classes from the source and target domains.}
\label{fig:explanation_of_attacks_b}
\end{wrapfigure}
In this experiment, we evaluate the effect of using poisoned data that looks like the source domain data. 
The poisoned data is generated by superimposing an image from the source domain with an image from the target domain.
This method of generating poison data is known as watermarking 
\cite{shafahi2018poison}.
To generate watermarked poison data we select an image from the target domain ($t$) and a base image from the source domain ($s$) such that it has the same class as the target domain image and lies closest to the target image (in the input space).  
The poisoned image ($p$) is obtained by a convex combination of the base and target images i.e., $p = \alpha t + (1 - \alpha) s$ where $\alpha \in [0, 1]$. 
$\alpha$ is selected such that the target image is not visible in the poison image ensuring the poisoned image looks like the image from the source domain. 
We use the same labeling function as discussed in the previous section to label the poisoned image and add 10\% poison data to the source. 
The illustrative picture of the effect of poisoning in this scenario is presented in Fig.~\ref{fig:explanation_of_attacks_b}. 
Successful poisoning, in this case, works just like in the previous experiment i.e., by making the representations of the data from wrong classes in source and target domains similar for DANN/CDAN, by reducing the discrepancy between the classifiers on target data for MCD and inducing a conflict between the supervised and auxiliary task for SSL.
The t-SNE embedding showing the effect of poisoning (Fig.~\ref{fig:exp_2}) and watermarked poison data (Fig.~\ref{fig:watermarking_poison}) are shown in the Appendix.
We evaluate the effectiveness of this method on the Digits dataset for different values of $\alpha$. The results in Table~\ref{Table:digits_experiment_2} show a significant decrease in the target domain accuracy even with a small watermarking percentage for all methods except CDAN. 
This is because the success of CDAN is dependent on the correctness of the pseudo-labels on the target domain data (output of the classifier), which are used in the discriminator. 
Correct pseudo-labels provide CDAN a positive reinforcement to align correct classes from the two domains,
leading to a failure of poisoning. 
However, as we increase the amount of watermarking, the quality of pseudo labels deteriorates. 
Thus, providing a negative reinforcement to CDAN that causes the alignment of wrong classes from the two domains.

\subsection{Poisoning using clean-label source and target domain data\label{sec:attack_3}}
\begin{wrapfigure}[11]{r}{0.27\textwidth}
\vspace{-0.9cm}
\centering
\includegraphics[width=0.25\columnwidth]{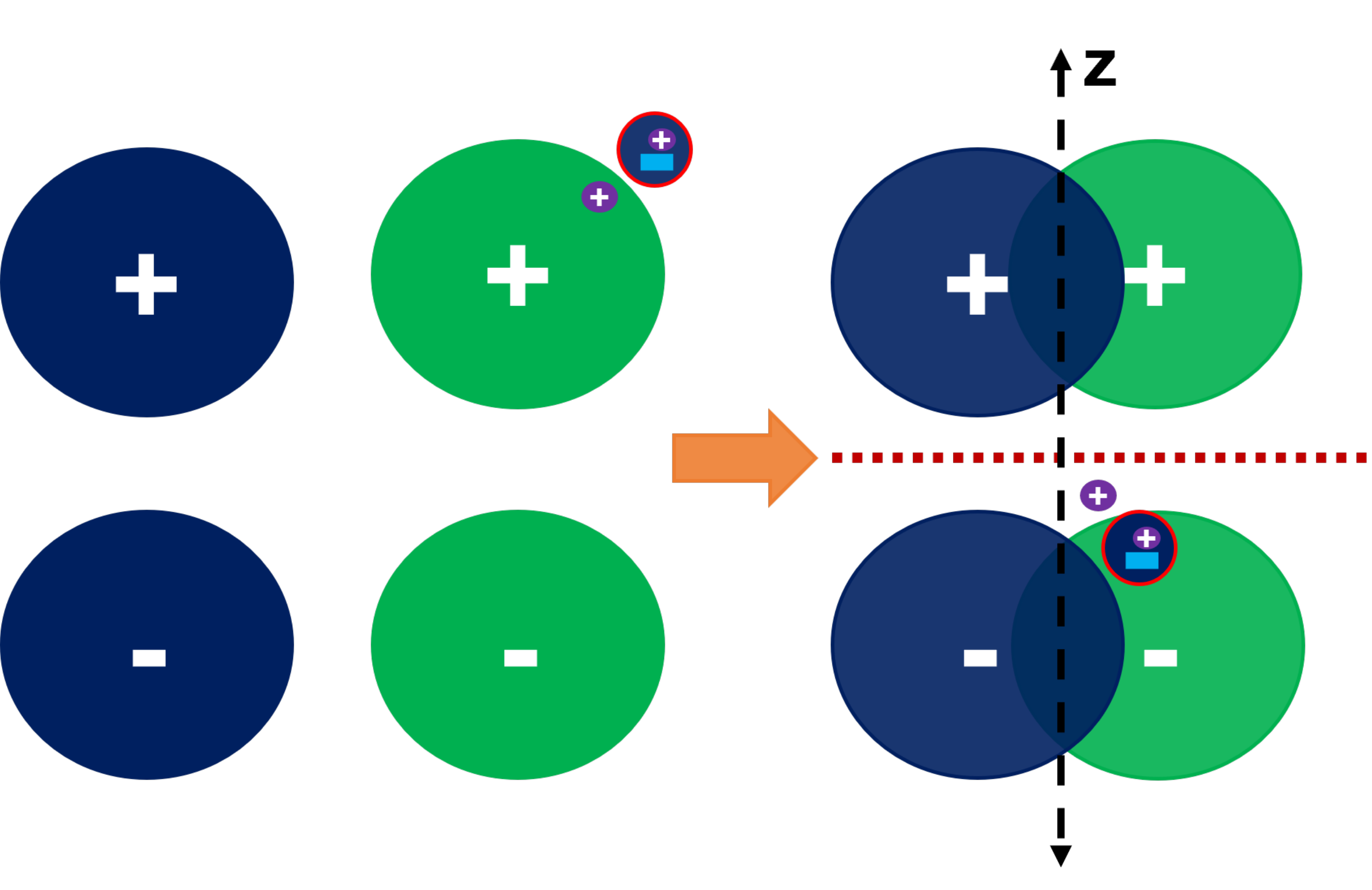}
\includegraphics[width=0.2\columnwidth]{images/Legend_0.pdf}
\vspace{-0.3cm}
\caption{Successful poisoning using clean-label poisoned data aligns the target point (purple) close to the wrong class (-).}
\label{fig:explanation_of_attacks_c}
\end{wrapfigure}
In this experiment, correctly labeled data is used for poisoning.
To generate clean-label poison data that can affect the performance of UDA methods we must affect the features of the poison data. This requires solving a bilevel optimization problem \cite{huang2020metapoison, mehra2019penalty, mehra2020robust} which we present in the Appendix~\ref{app:clean_label_bilevel}. 
Due to the high computational complexity involved in solving the bilevel problem, we propose to use a simple alternating optimization to demonstrate the feasibility of a clean label poisoning attack against UDA. 
We use the setting of previous works \cite{huang2020metapoison,shafahi2018poison} and consider misclassification of a single target domain test point ($x_\mathrm{test}^\mathrm{target}, y_\mathrm{test}^\mathrm{target}$) rather than affecting the accuracy of the entire target domain as done in the previous two experiments.
Let $u=\{u_1,...,u_n\}$ denote the poisoned data.
To ensure a clean label, each poison point $u_i$ must have a bounded perturbation from a base point
$x_i^\mathrm{base}$ i.e, $\|u_i - x_i^\mathrm{base}\| =\|\delta_i\| \leq \epsilon$ and has label of the base i.e., $y_i^\mathrm{base}$. 
Thus, $\hat{\mathcal{D}}^\mathrm{poison} = \{(u_i, y_i^\mathrm{base})\}_{i=1}^{N_{\mathrm{poison}}}$, $\hat{\mathcal{D}}_\mathrm{source} = \{(x_i^\mathrm{source}, y_i^\mathrm{source})\}_{i=1}^{N_\mathrm{source}}$ and $\hat{\mathcal{D}}_\mathrm{target} = \{(x_i^\mathrm{target}, y_i^\mathrm{target})\}_{i=1}^{N_\mathrm{target}}$.
The clean-label poison data $u$ is such that when the victim uses $\hat{\mathcal{D}}^\mathrm{source} \bigcup \hat{\mathcal{D}}^\mathrm{poison}$ and $\hat{\mathcal{D}}^\mathrm{target}$ for UDA, 
the target domain test point ($x_\mathrm{test}^\mathrm{target}, y_\mathrm{test}^\mathrm{target}$) is misclassified.
The optimization problem for the clean-label attack is as follows.
\begin{equation}
    \begin{split}
    \small
        \min_{u}\; \sum_{i=1}^{N_\mathrm{poison}}\|g(x_\mathrm{test}^\mathrm{target}; &\theta) - g(u_i; \theta)\|_2^2 \;\;\mathrm{s.t.}\;\; \|x^\mathrm{base}_{i} - u_i\| \leq \epsilon, \;\;\mathrm{for \; i=1,..., N_{poison}},\\
        &\min_{\theta}\; \mathcal{L}_\mathrm{UDA}(\hat{\mathcal{D}}^\mathrm{source} \bigcup \hat{\mathcal{D}}^\mathrm{poison}, \hat{\mathcal{D}}^\mathrm{target}; \theta).
    \end{split}
    \label{eq:alternating_clean_label}
\end{equation}

The first problem minimizes the distance between the representations of the poison and the target domain test data (first term) while ensuring the poison data is not too far from the base data (second term). The second problem optimizes the parameters of the representation using UDA methods.
\begin{wrapfigure}[12]{r}{0.3\textwidth}
\vspace{-0.3cm}
\centering
\includegraphics[width=0.3\columnwidth]{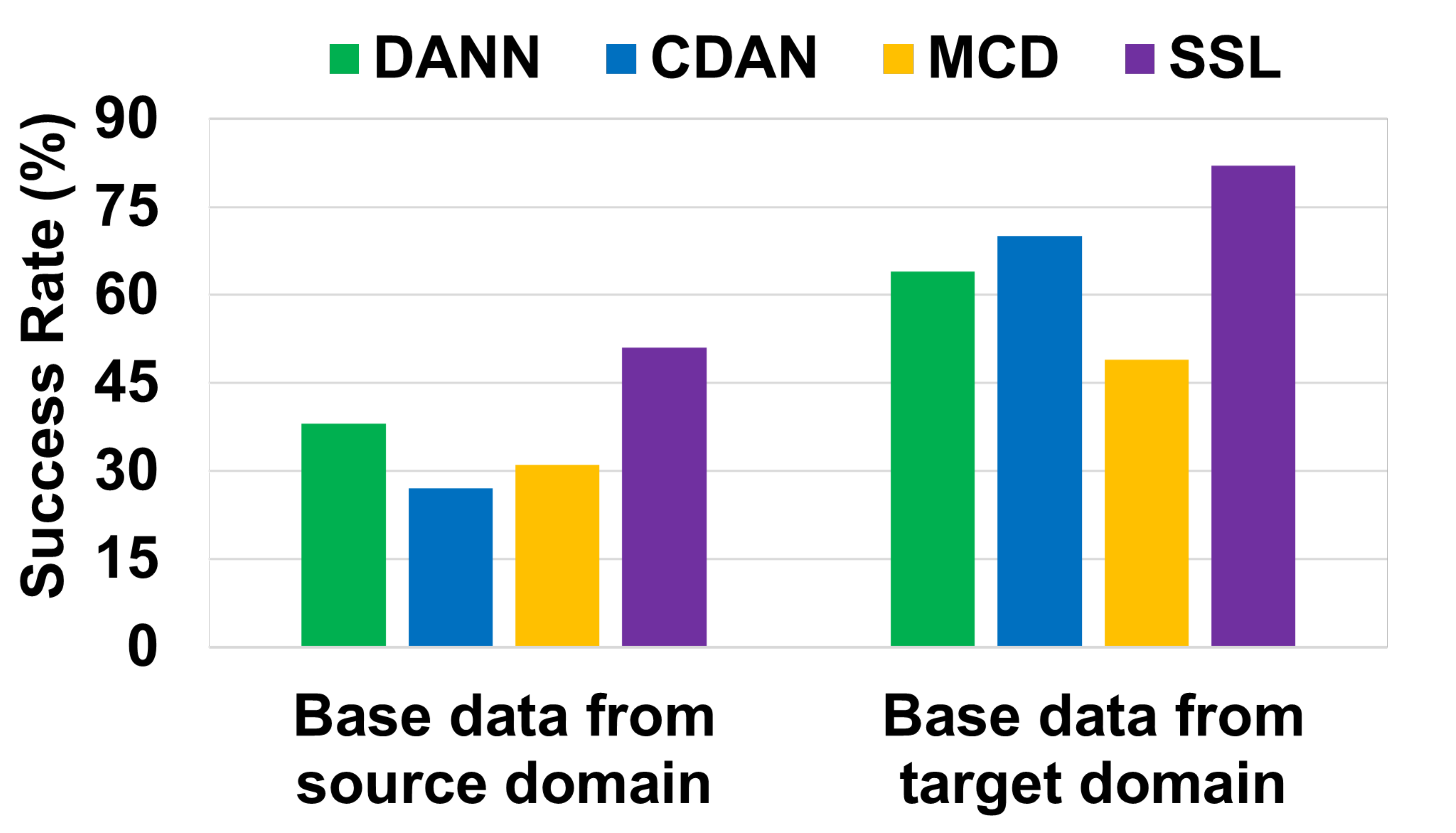}
\vspace{-0.6cm}
\caption{Attack success rate of clean-label poisoning using base data from source/target for a two-class problem in MNIST $\rightarrow{}$  MNIST\_M.}
\label{fig:clean_label_poison}
\end{wrapfigure}
Attack success is evaluated by solving the second problem in Eq.~\ref{eq:alternating_clean_label} from scratch and evaluating the classification of $x_\mathrm{test}^\mathrm{target}$. 
This is illustrated in Fig.~\ref{fig:explanation_of_attacks_c}. The left part shows the case before retraining using the poison data generated from Eq.~\ref{eq:alternating_clean_label} and the right part shows how poisoning induces misclassification. 
We use two approaches for poisoning. The first uses source domain data and the second uses target domain data as base data.  
We add 1\% poisoned data and test the effect of poisoning on a two-class (3 vs 8) domain adaptation problem on MNIST $\rightarrow{}$MNIST\_M (see Appendix~\ref{app:experiments_details}).
The results in Fig.~\ref{fig:clean_label_poison} show that using target domain data as base data is significantly more successful under small permissible perturbation ($\epsilon$). Using base data from the source domain requires larger distortion to keep the poison data close to the target point in the representation space and is hence less successful.
This shows the feasibility of clean label attacks against UDA methods. 
We believe the attack success can be improved by solving the bilevel level problem (Eq.~\ref{eq:bilevel_clean_label} in Appendix~\ref{app:clean_label_bilevel}) and is left as future work.

\section{Conclusion}
We studied the problem of UDA and highlighted the limitations of learning under this setting.
We proposed a simple lower bound on the target domain error for UDA, dependent on the labeling function induced by the representation. 
The lower bound demonstrated that learning a domain invariant representation while minimizing error on the source domain cannot guarantee good generalization on the target domain.
We analyzed a simple model and showed the existence of cases where UDA can naturally succeed or fail. 
The analysis also highlighted a case where, without access to any labeled target domain data, the success and the failure are equally likely.
In such a case, the presence of even a small amount of poisoned data can make the data distribution unfavorable for UDA methods,
making them fail 
dramatically in comparison to the case without poisoning. 
We proposed novel data poisoning attacks to demonstrate the failure of popular UDA methods with a small amount of poisoned data.
Our results suggest that the performance of a UDA method in presence of poisoned data indicates how well the method aligns the conditional distributions across the two domains. 
Thus, we believe our attacks can be used for evaluating UDA methods, beyond simple benchmark datasets, to reveal their robustness to data distributions inherently unfavorable for UDA.

\section{Acknowledgment}
We thank the anonymous reviewers for their insightful comments and suggestions. This work was supported by the NSF EPSCoR-Louisiana Materials Design Alliance (LAMDA) program \#OIA-1946231 and by LLNL Laboratory Directed Research and Development project 20-ER-014 (LLNL-CONF-824206). 
This work was performed under the auspices of the U.S. Department of Energy by the Lawrence Livermore National Laboratory under Contract No. DE-AC52-07NA27344, Lawrence Livermore National Security, LLC
\footnote{This document was prepared as an account of the work sponsored by an agency of the United States Government. Neither the United States Government nor Lawrence Livermore National Security, LLC, nor any of their employees make any warranty, expressed or implied, or assumes any legal liability or responsibility for the accuracy, completeness, or usefulness of any information, apparatus, product, or process disclosed, or represents that its use would not infringe privately owned rights. Reference herein to any specific commercial product, process, or service by trade name, trademark, manufacturer, or otherwise does not necessarily constitute or imply its endorsement, recommendation, or favoring by the United States Government or Lawrence Livermore National Security, LLC. The views and opinions of the authors expressed herein do not necessarily state or reflect those of the United States Government or Lawrence Livermore National Security, LLC, and shall not be used for advertising or product endorsement purposes.}.  


\bibliographystyle{plain}
\bibliography{neurips_2021}


\medskip

\section*{Checklist}

\begin{enumerate}

\item For all authors...
\begin{enumerate}
  \item Do the main claims made in the abstract and introduction accurately reflect the paper's contributions and scope?
    \answerYes{See Sec.~\ref{sec:analysis} and~\ref{sec:experiments}.}
  \item Did you describe the limitations of your work?
    \answerYes{See Sec.~\ref{sec:analysis}.}
  \item Did you discuss any potential negative societal impacts of your work?
    \answerNA{}
  \item Have you read the ethics review guidelines and ensured that your paper conforms to them?
    \answerYes{}{}
\end{enumerate}

\item If you are including theoretical results...
\begin{enumerate}
  \item Did you state the full set of assumptions of all theoretical results?
    \answerYes{See Sec.~\ref{sec:analysis} and Appendix~\ref{app:analysis}.}
	\item Did you include complete proofs of all theoretical results?
    \answerYes{See the Appendix.}
\end{enumerate}

\item If you ran experiments...
\begin{enumerate}
  \item Did you include the code, data, and instructions needed to reproduce the main experimental results (either in the supplemental material or as a URL)?
    \answerYes{See Appendix~\ref{app:experiments_details}.}
  \item Did you specify all the training details (e.g., data splits, hyperparameters, how they were chosen)?
    \answerYes{See Appendix~\ref{app:experiments_details}.}
	\item Did you report error bars (e.g., with respect to the random seed after running experiments multiple times)?
    \answerYes{}
	\item Did you include the total amount of compute and the type of resources used (e.g., type of GPUs, internal cluster, or cloud provider)?
    \answerYes{See Appendix~\ref{app:experiments_details}.}
\end{enumerate}

\item If you are using existing assets (e.g., code, data, models) or curating/releasing new assets...
\begin{enumerate}
  \item If your work uses existing assets, did you cite the creators?
    \answerYes{}
  \item Did you mention the license of the assets?
    \answerNA{}
  \item Did you include any new assets either in the supplemental material or as a URL?
    \answerNA{}
  \item Did you discuss whether and how consent was obtained from people whose data you're using/curating?
    \answerNA{}
  \item Did you discuss whether the data you are using/curating contains personally identifiable information or offensive content?
    \answerNA{}
\end{enumerate}

\item If you used crowdsourcing or conducted research with human subjects...
\begin{enumerate}
  \item Did you include the full text of instructions given to participants and screenshots, if applicable?
    \answerNA{}
  \item Did you describe any potential participant risks, with links to Institutional Review Board (IRB) approvals, if applicable?
    \answerNA{}
  \item Did you include the estimated hourly wage paid to participants and the total amount spent on participant compensation?
    \answerNA{}
\end{enumerate}

\end{enumerate}

\clearpage
\appendix
\begin{center}
{\LARGE \bf Appendix}
\end{center}
We present the proof of Theorem~\ref{thm:lower_bound} in Appendix~\ref{app:analysis} followed by the analysis of the illustrative cases of UDA failure in Appendix~\ref{app:UDA_cases}. Then we present the results for the experiment of using different poison percentages when mislabeled data is used for poisoning in Appendix~\ref{app:poison_percentage} followed by the proposed bilevel formulation for clean label attacks in Appendix~\ref{app:clean_label_bilevel}. We discuss additional related work in Appendix~\ref{app:additional_related_work} and present additional experiments in Appendix~\ref{app:additional_experiments}. We conclude in Appendix~\ref{app:experiments_details} by providing the details of the datasets used, model architectures, and the clean label experiment. 

\section{Proof of the lower bound on the target domain loss}\label{app:analysis}
\begin{theoremrep}[\ref{thm:lower_bound}]
Let $\mathcal{H}$ be the hypothesis class and $\mathcal{G}$ be the class  representation maps. Then, for all $h\in \mathcal{H}$ and $g \in \mathcal{G}$,
\[
e_T(h) \geq \max\{e_S(\tilde{f}_S,\tilde{f}_T),e_T(\tilde{f}_S,\tilde{f}_T)\}  - e_S(h) - D_1(\tilde{p}_S,\tilde{p}_T).
\]
\end{theoremrep}

\begin{proof}
\begin{eqnarray*}
e_S(h) + e_T(h) 
&=& e_T(h,\tilde{f}_T)) + e_T(h,\tilde{f}_S) + e_S(h,\tilde{f}_S) - e_T(h,\tilde{f}_S)\\
&\geq & e_T(\tilde{f}_S,\tilde{f}_T) + e_S(h,\tilde{f}_S) - e_T(h,\tilde{f}_S)\\
&=& e_T(\tilde{f}_S,\tilde{f}_T) + \int (\tilde{p}_S(z)-\tilde{p}_T(z))|h(z)-\tilde{f}_S(z)|\;dz\\
&\geq& e_T(\tilde{f}_S,\tilde{f}_T) - \int |\tilde{p}_S(z)-\tilde{p}_T(z)|\;dz\\
&=& e_T(\tilde{f}_S,\tilde{f}_T) - D_1(\tilde{p}_S,\tilde{p}_T).
\end{eqnarray*}
Similarly, we can also show that 
\[
e_S(h) + e_T(h) \geq e_S(\tilde{f}_S,\tilde{f}_T) - D_1(\tilde{p}_S,\tilde{p}_T).
\]
Combining the two results gives us the statement of the theorem.
\end{proof}

\begin{corollaryrep}[\ref{cor:two}]
For all $h\in \mathcal{H}$ and $g \in \mathcal{G}$,
\begin{equation*}
|e_T(h) - {e_S(\tilde{f}_S,\tilde{f}_T)| \leq e_S(h) + D_1(\tilde{p}_S,\tilde{p}_T)},\;\;\mathrm{and}\;\;
|e_T(h) - {e_T(\tilde{f}_S,\tilde{f}_T)| \leq e_S(h) + D_1(\tilde{p}_S,\tilde{p}_T)}.
\end{equation*}
\end{corollaryrep}
\begin{proof}
From the upper bound we have,
\begin{eqnarray*}
e_T(h) - e_S(\tilde{f}_S,\tilde{f}_T)) \leq e_S(h) + D_1(\tilde{p}_S,\tilde{p}_T) \;\;\mathrm{and}\;\; 
e_T(h) - e_T(\tilde{f}_S,\tilde{f}_T)) \leq e_S(h) + D_1(\tilde{p}_S,\tilde{p}_T)
\end{eqnarray*}

From the lower bound (Eq.~\ref{eq:lower bound}) we have,
\begin{eqnarray*}
 e_T(h) - e_S(\tilde{f}_S,\tilde{f}_T) \geq -e_S(h)  - D_1(\tilde{p}_S,\tilde{p}_T) \;\;\mathrm{and}\;\; 
 e_T(h) - e_T(\tilde{f}_S,\tilde{f}_T) \geq -e_S(h)  - D_1(\tilde{p}_S,\tilde{p}_T)
\end{eqnarray*}
Combining the results from the upper and the lower bounds gives us the statement of the corollary.
\end{proof}

\section{Illustrative examples of UDA failure}\label{app:UDA_cases}
In this section, we provide the details of the analysis of the illustrative cases in the main paper.
As described in Sec.~\ref{sec:uda_cases}, the input space $\mathcal{X}$ is in $\mathbb{R}^2$ and
the source and the target distributions are Gaussian mixtures
\[
p_S(x) = 0.5 p_{S+}(x) + 0.5 p_{S-}(x)\;\;\mathrm{and}\;\;p_T(x) = 0.5 p_{T+}(x) + 0.5 p_{T-}(x),
\]
where 
$p_{S+}(x)=\mathcal{N}(\mu_{S+},\sigma^2 I)$, 
$p_{S-}(x)=\mathcal{N}(\mu_{S-},\sigma^2 I)$, 
$p_{T+}(x)=\mathcal{N}(\mu_{T+},\sigma^2 I)$, and
$p_{T-}(x)=\mathcal{N}(\mu_{T-},\sigma^2 I)$.
The true labeling function $f(x)$ in the input space is assumed linear:
$f(x) = I[v^Tx >0]$ where $v$ is the unit normal vector to the decision boundary.
The representation space $\mathcal{Z}$ is in $\mathbb{R}$ and the representation map $g:\mathcal{X} \to\mathcal{Z}$ is linear:  $g(x)=u^Tx$ where $\|u\|=1$.
For the hypothesis, we use $h(z)=\Phi(a z + b)$ which is a linear model $az+b$ followed by a saturating function which can be the cumulative normal distribution $\Phi$ (or others such as the logistic function $l$).

The representation map $g$ induces the distributions $\tilde{p}(z)$ over $\mathcal{Z}$ as
\[
\tilde{p}_S(z)=0.5\mathcal{N}(u^T\mu_{S+},\sigma^2)+0.5\mathcal{N}(u^T\mu_{S-},\sigma^2),\;\;\mathrm{and}
\]
\[
\tilde{p}_T(z)=0.5\mathcal{N}(u^T\mu_{T+},\sigma^2)+0.5\mathcal{N}(u^T\mu_{T-},\sigma^2).
\]
The map $g$ also induces the labeling function $\tilde{f}(z)$ on $\mathcal{Z}$ defined as $\tilde{f}(z)=E_{\mathcal{D}}[f(x)|g(x)=z]$ \cite{ben2007analysis}.
Computing this quantity can be complex in general but is relatively straightforward for a mixture of Gaussians and a simple half-space labeling function $f(x)$.
Following the definition, we have 
\begin{equation}\label{eq:induced_labeling_function}
\tilde{f}(z) = E_{\mathcal{D}}[f(x)|g(x)=z] = \int_{\mathcal{Z}} f(x)\;I[u^Tx = z]\;p(x|z=g(x)) dx.
\end{equation}
In our example, the integral $\int_{\mathbb{R}^2}\;\cdot\; dx$ can be decomposed into $\int_{-\infty}^{\infty} \int_{-\infty}^{\infty}\;\cdot\;dz dw$ where $z$ and $w$ are the coordinates along the rotated axes $u$ and $u^\perp$ (see Fig.~\ref{fig:failure_proof}).

\begin{figure}[tb]
  \centering
    {\includegraphics[width=.5\columnwidth]{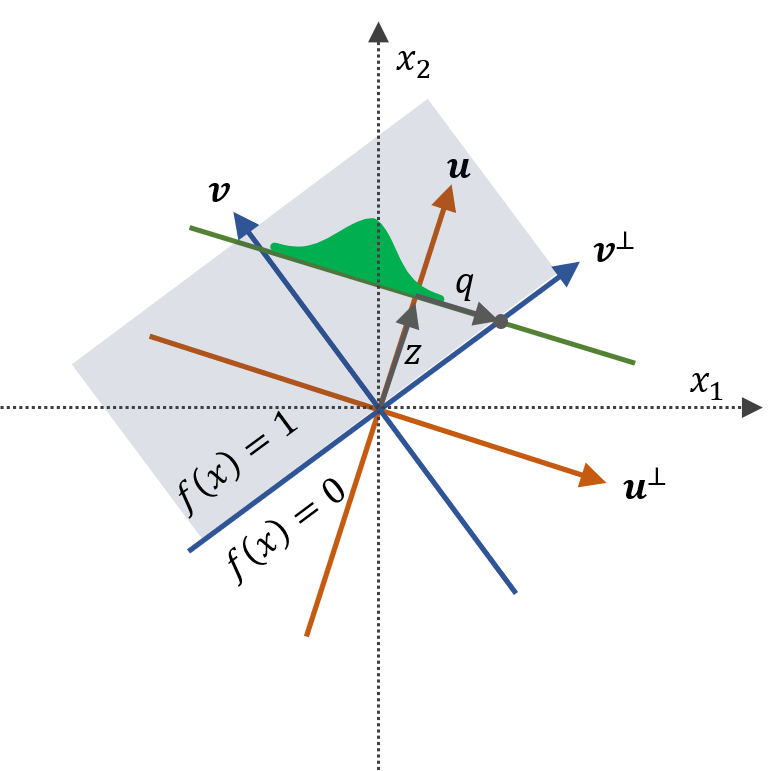}}  
  \caption{This figure provides a visual help for deriving the induced labeling function $\tilde{f}(z)$ in Eq.~\ref{eq:induced_labeling_function}. $u$ is the direction of the 1-D projection $g(z)=u^Tx$, $v$ is the direction of the labeling function $f(x)$ in the input space where we assumed $f(x)=I[v^Tx > 0]$, and $q$ is the intersection of the two lines $v^Tx=0$ and $u^T(x-z u)=0$ projected along the $u^\perp$ direction. Evaluating Eq.~\ref{eq:induced_labeling_function_intermediate} using the help of this figure results in Eq.~\ref{eq:induced_labeling_function_final}. 
  }
  \label{fig:failure_proof}
\end{figure}
We therefore have
\begin{eqnarray}
\tilde{f}(z) &=& \int_{-\infty}^{\infty} \int_{-\infty}^{\infty} I[v^Tx>0]\;I[u^Tx = z]\;p(x|z=g(x)) dz dw \nonumber\\ 
&=& \int_{-\infty}^{\infty} I[v^Tx>0]\;(0.5\mathcal{N}_w(\mu_{+}^Tu^{\perp},\sigma^2I) + 0.5\mathcal{N}_w(\mu_{+}^Tu^{\perp},\sigma^2I)) dw. \label{eq:induced_labeling_function_intermediate}
\end{eqnarray}
This integral can be evaluated as 
\begin{equation}\label{eq:induced_labeling_function_final}
\tilde{f}(z) = \left(
\begin{array}{ll}
0.5 \Phi(\frac{q-\mu_{+}^Tu^\perp}{\sigma}) + 0.5 \Phi(\frac{q-\mu_{+}^Tu^\perp}{\sigma}) & \textrm{if}\;\; v^Tu^\perp<0 \\
0.5 [1-\Phi(\frac{q-\mu_{+}^Tu^\perp}{\sigma})] + 0.5 [1-\Phi(\frac{q-\mu_{+}^Tu^\perp}{\sigma})] & \textrm{if}\;\; v^Tu^\perp>0 \\
0.5(1+sign(z)) & \textrm{if}\;\; v^Tu^\perp=0\;\;\textrm{and}\;\;v^Tu>0 \\
0.5(1-sign(z)) & \textrm{if}\;\; v^Tu^\perp=0\;\;\textrm{and}\;\;v^Tu<0 
\end{array}
\right.
\end{equation}
where $\Phi$ is the cumulative normal distribution and $q$ is the intersection of the two lines $v^Tx=0$ and $u^T(x-z u)=0$ projected along the $u^\perp$ direction. More concretely, 
\[
q = \frac{u_1v_1 + u_2 v_2}{u_1 v_2 - u_2 v_1}z
\]
where $u=[u_1,u_2]^T$ and $v=[v_1,v_2]^T$.

The UDA minimization problem is 
\begin{equation}\label{eq:example_min}
\min_{u,a,b}\; e_S(h) + \lambda D(\tilde{p}_S,\tilde{p}_T) + \eta (\|u\|^2-1)^2,
\end{equation}
where the last term was added to enforce $\|u\|=1$.
For differentiability, we consider the squared loss instead of the absolute loss:
\[
e_S(h) = E_S[(\Phi(az+b) - \tilde{f}_S(z))^2] = \int_{\mathbb{R}} \tilde{p}_S(z) \left(\Phi(az+b)-\tilde{f}_S(z)\right)^2 dz
\]
and also 
\[
D(\tilde{p},\tilde{p}') = \int_{\mathbb{R}} (\tilde{p}(z)-\tilde{p}'(z))^2 dz.
\]
The expectation in $e_S(h)$ can only be computed numerically due to the complex formula for $\tilde{f}(z)$.
On the other hand, the mismatch loss is
\begin{eqnarray*}
D(\tilde{p}_S(z),\tilde{p}_T(z))&=&\int_{\mathbb{R}} (\tilde{p}_S(z) - \tilde{p}_T(z))^2 dz\\
&=&\int_{\mathbb{R}} \left(\frac{0.5}{2\pi \sigma^2}\right)^2\left[e^{-\frac{(z-u^T\mu_{S+})^2}{2\sigma^2}}
+e^{-\frac{(z-u^T\mu_{S-})^2}{2\sigma^2}}-e^{-\frac{(z-u^T\mu_{T+})^2}{2\sigma^2}}-e^{-\frac{(z-u^T\mu_{T-})^2}{2\sigma^2}}\right]^2 dz,
\end{eqnarray*}
which can be computed either numerically or analytically. 

The three cases explained in the main paper are as follows:
\begin{itemize}
    \item[Case 1]: $\mu_{S+}=[-1,1]^T$, $\mu_{S-}=[-1,-1]^T$, $\mu_{T+}=[1,1]^T$, $\mu_{T-}=[1,-1]^T$, $v_S(x) = v_T(x) = [0,1]^T$, $\lambda=10^{-1}$
    \item[Case 2]: $\mu_{S+}=[-1,1]^T$, $\mu_{S-}=[-1,-1]^T$, $\mu_{T+}=[1,-1]^T$, $\mu_{T-}=[1,1]^T$, $v_S(x) = -v_T(x) = [0,1]^T$, $\lambda=10^{-1}$
    \item[Case 3]: $\mu_{S+}=[0,1]^T$, $\mu_{S-}=[0,-1]^T$, $\mu_{T+}=[-1,0]^T$, $\mu_{T-}=[1,0]^T$, $v_S(x) = [0,1]^T$, $v_T(x) = [-1,0]^T$ , $\lambda=10^{-2}$
\end{itemize}
The other shared parameters are $\sigma=1$ and $\eta=10$. The $\lambda$ determines the optimal tradeoff between $E_s$ and $D$ in Eq.~\ref{eq:example_min}.

We solve Eq.~\ref{eq:example_min} numerically using \emph{scipy.optimize.minimize(method=`Nelder-Mead')} function which is stable even if the cost function may be non-differentiable. Starting from random initial conditions and running until convergence, the solution $u$ for both Case 1 and Case 2 converges to $[0,1]^T$. 

For Case 1 (favorable case), we get $\max\{e_S(\tilde{f}_S,\tilde{f}_T),e_T(\tilde{f}_S,\tilde{f}_T)\} < 10^{-3}$ and $e_T(h)<10^{-3}$ which shows UDA was successful.

For Case 2 (unfavorable case), we get $\max\{e_S(\tilde{f}_S,\tilde{f}_T),e_T(\tilde{f}_S,\tilde{f}_T)\}>0.99$ and $e_T(h)>0.99$ which shows UDA was unsuccessful.

For Case 3 (ambiguous case), there is an almost equal chance of 
$u$ converging to $[-0.70,0.72]^T$ or $[0.70,0.72]^T$.
For the former, we get $\max\{e_S(\tilde{f}_S,\tilde{f}_T),e_S(\tilde{f}_S,\tilde{f}_T)\}<10^{-4}$ and $e_T(h)<10^{-3}$ where UDA is successful.
For the latter, we get $\max\{e_S(\tilde{f}_S,\tilde{f}_T),e_S(\tilde{f}_S,\tilde{f}_T)\}\approxeq 0.33$ and $e_T(h) \approxeq 0.33$ where UDA has failed.

\section{Details of the figures explaining the effect of poisoning on UDA methods}\label{app:fig_explanation}
As illustrated in Case 3 of Fig.~\ref{fig:different_cases_for_uda} in Sec.~\ref{sec:uda_cases}, UDA methods can be fooled into producing a representation that causes a large error on the target domain with a small amount of poisoned data. 
The simplest successful poisoning attack to fool UDA methods was shown in Sec.~\ref{sec:attack_1} (wrong-label incorrect-domain poisoning).
In this attack, we added mislabeled data (wrong-label) from the target domain into the source data (incorrect-domain). The left part of Fig.~\ref{fig:explanation_of_attacks_a} shows this setting. 
The right part of Fig.~\ref{fig:explanation_of_attacks_a} shows how the representation learned from discriminator-based UDA methods aligns the incorrect classes closer than the correct ones. Due to the lack of target domain labels, the loss of the discriminator is minimized as long as the green and blue blobs align, regardless of their labels. But to minimize the classification loss on the source domain the representation must classify the poison data correctly. This forces the representation of wrong source and target domain classes to be closer than the correct ones. As a result of this, the source classification loss and domain mismatch loss are minimized but the learned representation still incurs a large target domain error. This is exactly what Case 3 (right) of Fig.~\ref{fig:different_cases_for_uda} illustrated. The t-SNE embeddings in Fig.~\ref{fig:exp_1_a} confirm this on real datasets where representations are learned using popular discriminator-based UDA methods.

To make the poisoning attacks harder to detect we used watermarking-based attacks using poison data that has some features of the target data but still looks like the source data (Fig.~\ref{fig:watermarking_poison}). This setting is illustrated in the left part of Fig.~\ref{fig:explanation_of_attacks_b}. The right part of Fig.~\ref{fig:explanation_of_attacks_b} shows how discriminator-based UDA methods are fooled into producing a representation that fails to generalize on the target domain. Similar to the previous case of wrong-label incorrect-domain poisoning discriminator is optimal when the green and blue blobs align. 
However, since the poison data has incorrect labels the source classification loss prefers to align it with the wrong class (poison data labeled as + is aligned with source class with label +). 
Due to the presence of target features in the poison data (due to watermarking) the representation moves the target domain data closer to the poison data leading to an alignment of wrong source and target domain classes. As the percentage of watermarking increases the poisoning attack becomes more successful (Table~\ref{Table:digits_experiment_2}) indicating that target domain data is being aligned to wrong source domain classes similar to the poison data. Fig.~\ref{fig:exp_2} ((a) and (b)) demonstrate this effect on popular discriminator-based UDA methods.  

To make our attack even stealthier, we consider the effect of using correctly labeled poisoned data on the UDA methods. We generate such poison data by solving a computationally efficient version of the bilevel problem (Eq.~\ref{eq:bilevel_clean_label}), as shown in Eq.~\ref{eq:alternating_clean_label}. The left part of Fig.~\ref{fig:explanation_of_attacks_c} illustrates the poison data generated by solving Eq.~\ref{eq:alternating_clean_label}. The poison data specifically targets a particular target domain test point (shown in purple). Since the poison data is close in the representation space to the target domain test point, UDA methods align it closer to the class of the poison data. This leads to misclassification of the test point. Since the poison data is crafted for a specific test point, they don't have much effect on the entire target domain data. The right part of Fig.~\ref{fig:explanation_of_attacks_c} illustrates this effect and shows that clean labeled poison data can also successfully hurt the performance of UDA methods. 

\section{Effect of poison percentage on attack success with mislabeled poison data}\label{app:poison_percentage}
In this section, we evaluate the effect of using different poison percentages on attack success when mislabeled data is used for poisoning. As can be seen in Tables~\ref{Table:digits_experiment_1} and~\ref{Table:office31_experiment_1}, the success of wrong-label clean-domain poisoning with 10\% poisoned data is very limited. Thus, here we only focus on using a smaller poison percentage to study the attack success of wrong-label wrong-domain poisoning.
The results of the experiment are summarized in Table~\ref{Table:exp1_poison_percentage}. 
For all tasks, the presence of only 6\% poison data causes a significant decrease in the target domain accuracy. 
When the poison percentage is decreased further to only 2\% we still see a drop of at least 20\% in the target accuracy for all methods except CDAN\cite{long2017conditional}. 
The use of a conditional discriminator provides CDAN this robustness. 
However, the success of CDAN is dependent on the quality of the pseudo-labels from the classifier on the target domain data.
Good pseudo-labels provide CDAN a positive reinforcement to align correct source and target domain classes. 
Thus, leading to a failure of poisoning. 
However, as the percentage of poisoned data increases, the classifier begins to easily classify the target domain data into labels intended by the attacker, deteriorating the quality of the pseudo-labels. This provides a negative reinforcement to CDAN causing it to align wrong classes from the source and the target domain. As a result, the poisoning attack becomes successful. Thus, for wrong-label wrong-domain poisoning, increasing the percentage of poison data gradually drives UDA methods from the case favorable to UDA to the unfavorable one.

\begin{table}[tb]
  \caption{Effect of using different percentages of wrong-label incorrect-domain poisoned data on the target domain accuracy when training UDA methods on poisoned source domain data on the Digits tasks (mean$\pm$s.d. of 5 trials).}
  \label{Table:exp1_poison_percentage}
  \centering
  \small
  \resizebox{0.99\columnwidth}{!}{
    \begin{tabular}{c|cc|cc|cc|cc}
    \toprule
    \multirow{2}{*}{Poison$_\mathrm{target}$ (\%)} & \multicolumn{2}{c|}{\makecell{DANN}} & \multicolumn{2}{c|}{\makecell{CDAN}} & \multicolumn{2}{c|}{\makecell{MCD}} & \multicolumn{2}{c}{\makecell{SSL}}\\
    & MNIST $\rightarrow{}$ USPS & USPS $\rightarrow{}$ MNIST & MNIST $\rightarrow{}$ USPS & USPS $\rightarrow{}$ MNIST & MNIST $\rightarrow{}$ USPS & USPS $\rightarrow{}$ MNIST & MNIST $\rightarrow{}$ USPS & USPS $\rightarrow{}$ MNIST \\
    \midrule
    
    0\% (Clean) & 92.17$\pm$0.73 & 92.73$\pm$0.71 & 93.92$\pm$0.97 & 95.94$\pm$0.71 & 89.96$\pm$2.04 & 88.34$\pm$2.50 & 88.69$\pm$1.28 & 82.23$\pm$1.59\\
    \cmidrule{1-9}
    2\% & 63.53$\pm$2.09 & 94.72$\pm$0.63 & 90.54$\pm$0.91 & 88.79$\pm$2.34 & 22.74$\pm$2.17 & 51.02$\pm$3.57 & 65.88$\pm$2.93 & 41.25$\pm$2.32 \\
    
    4\% & 28.39$\pm$4.78 & 34.25$\pm$9.53 & 90.22$\pm$0.74 & 76.55$\pm$2.25 & 2.37$\pm$1.41 & 16.66$\pm$4.73 & 30.82$\pm$1.28 & 28.60$\pm$2.16\\
    
    6\% & 7.32$\pm$4.78 & 12.96$\pm$7.33 & 42.86$\pm$5.09 & 8.61$\pm$4.77 & 2.56$\pm$0.97 & 4.64$\pm$1.34 & 21.29$\pm$2.51 & 18.89$\pm$1.11 \\
    
    8\% & 0.97$\pm$0.44 & 1.63$\pm$0.41 & 7.02$\pm$3.88 & 5.35$\pm$0.94 & 7.04$\pm$0.25 & 4.43$\pm$1.76 & 10.84$\pm$1.52 & 11.11$\pm$2.74 \\
    
    10\% & 0.97$\pm$0.53 & 5.83$\pm$0.82 & 1.92$\pm$0.42 & 2.96$\pm$0.71 & 0.66$\pm$0.16 & 2.07$\pm$0.69 & 7.76$\pm$1.52 & 9.88$\pm$1.07 \\

    \bottomrule
    \end{tabular}
    }
\end{table}

\begin{figure}[tb]
  \centering
  \subfigure[DANN]{\includegraphics[width=0.24\columnwidth]{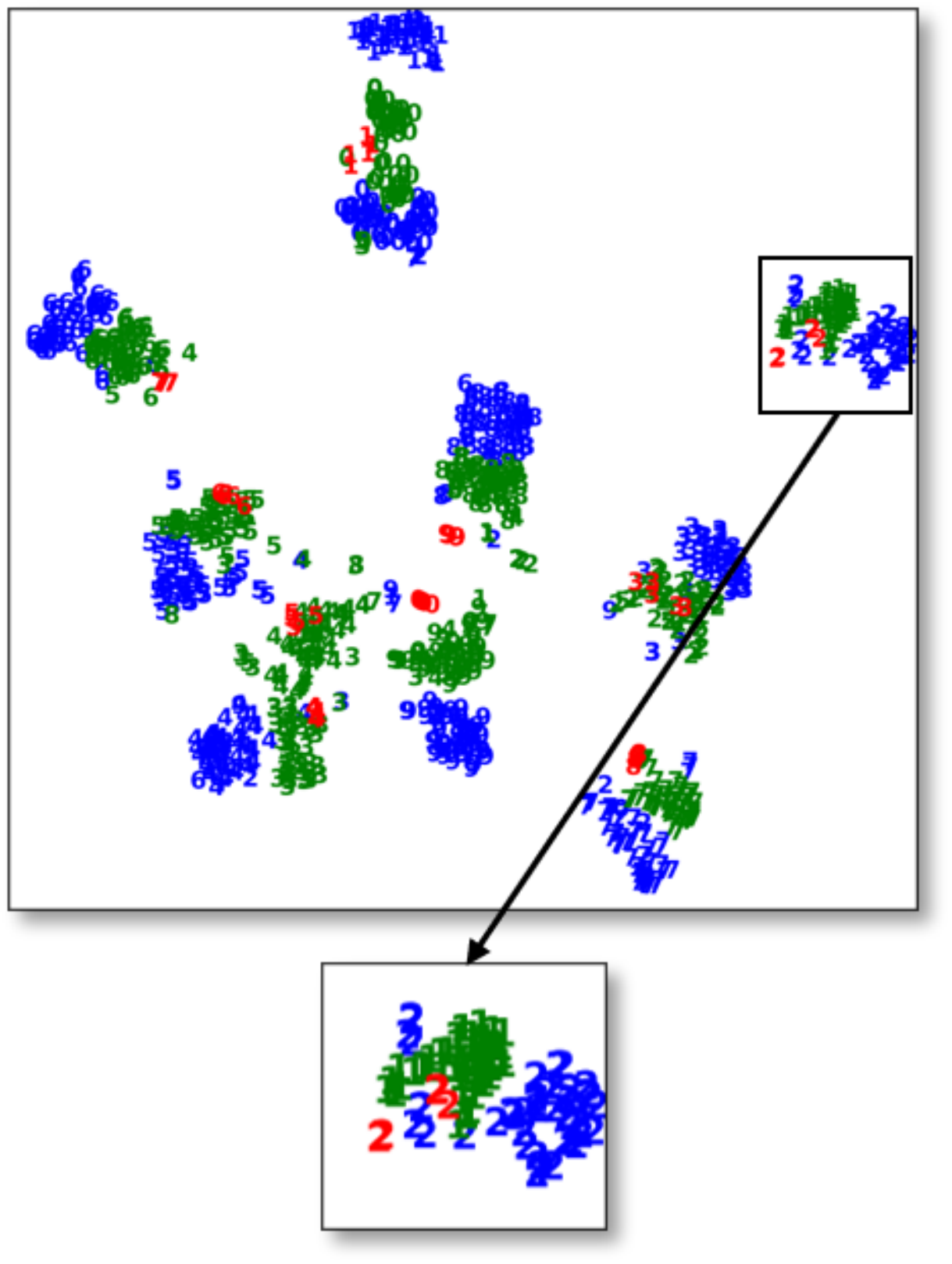}}
  \subfigure[CDAN]{\includegraphics[width=0.24\columnwidth]{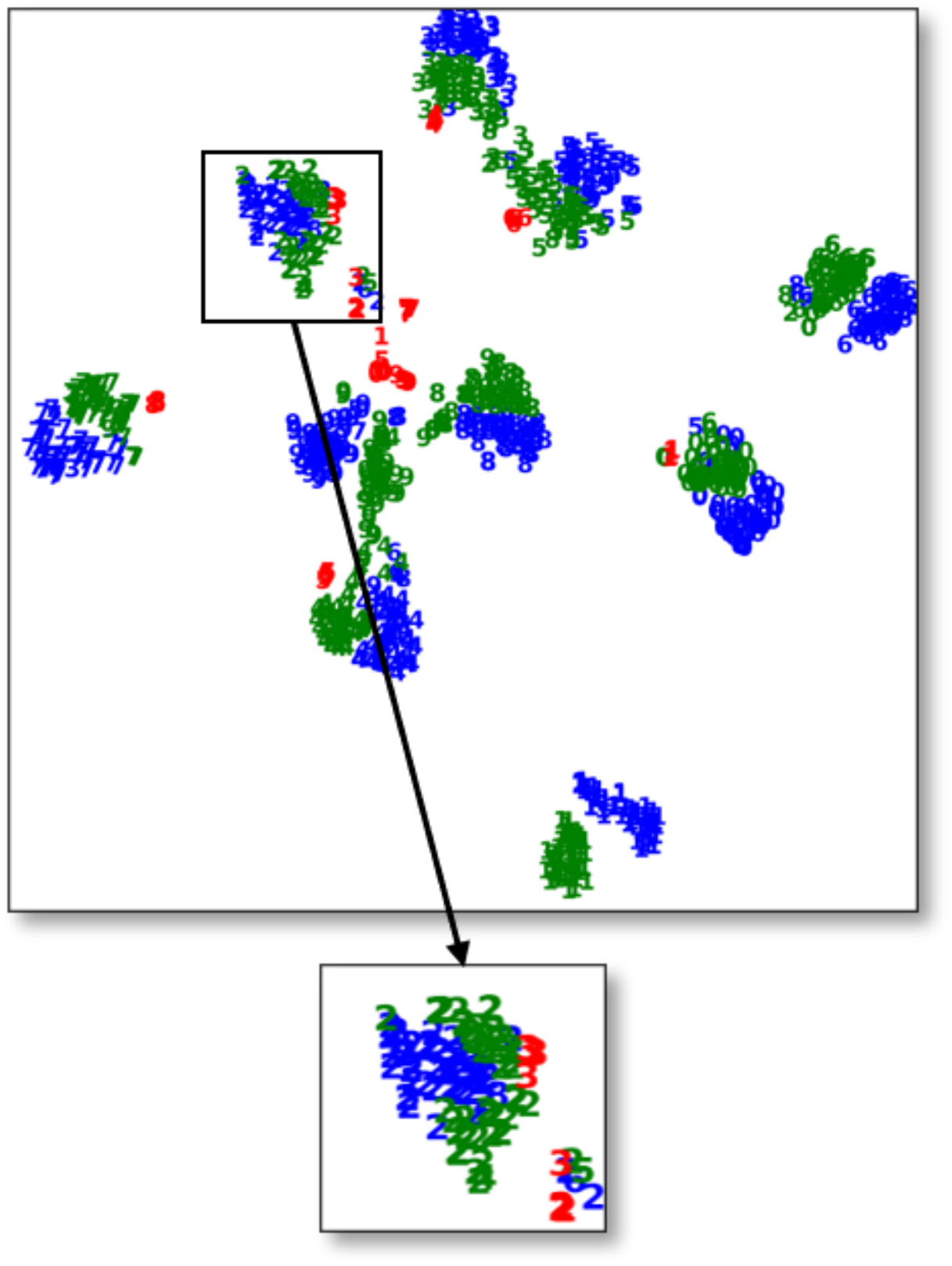}}
  \subfigure[MCD]{\includegraphics[width=0.24\columnwidth]{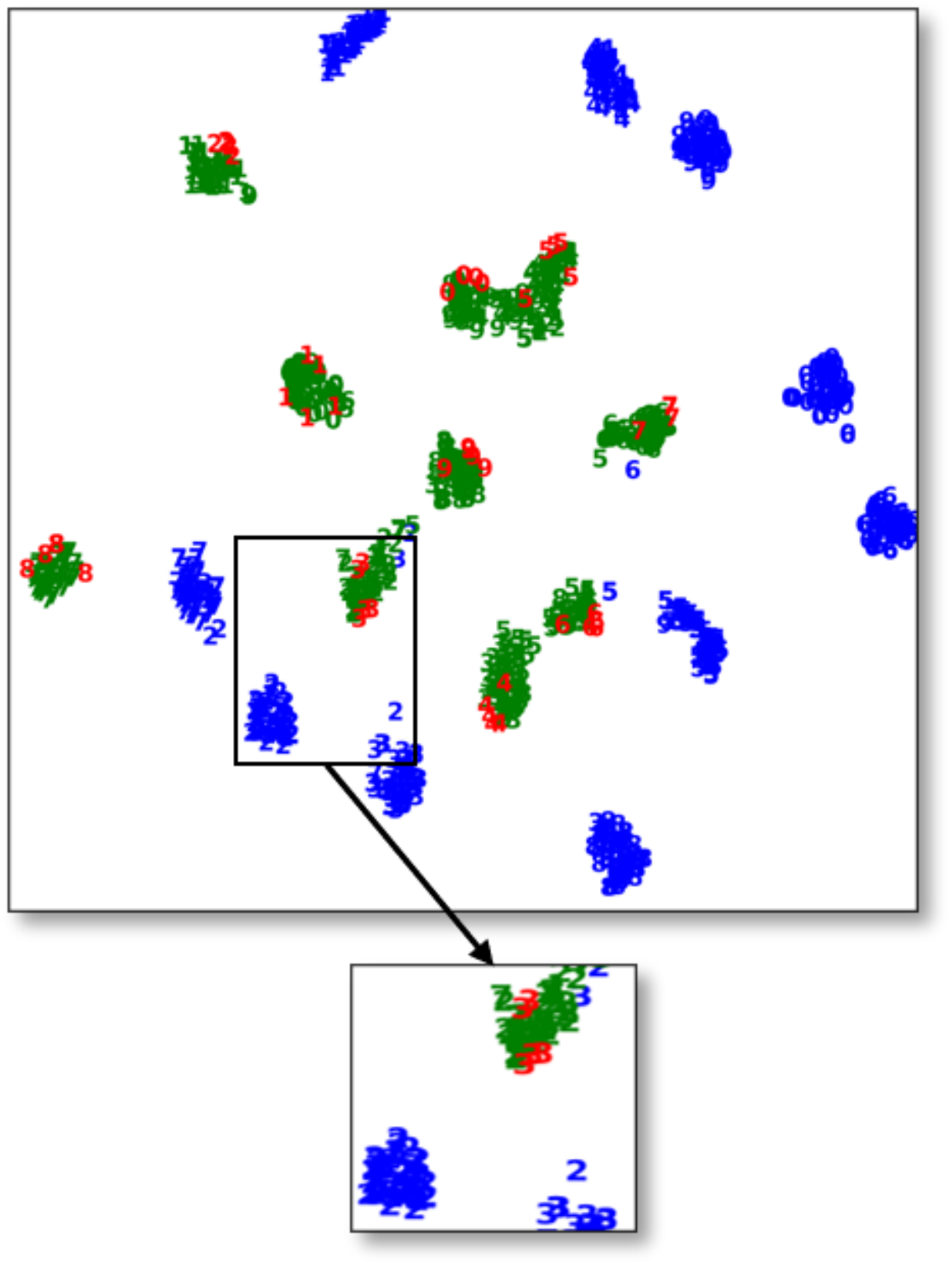}}
  \subfigure[SSL]{\includegraphics[width=0.24\columnwidth]{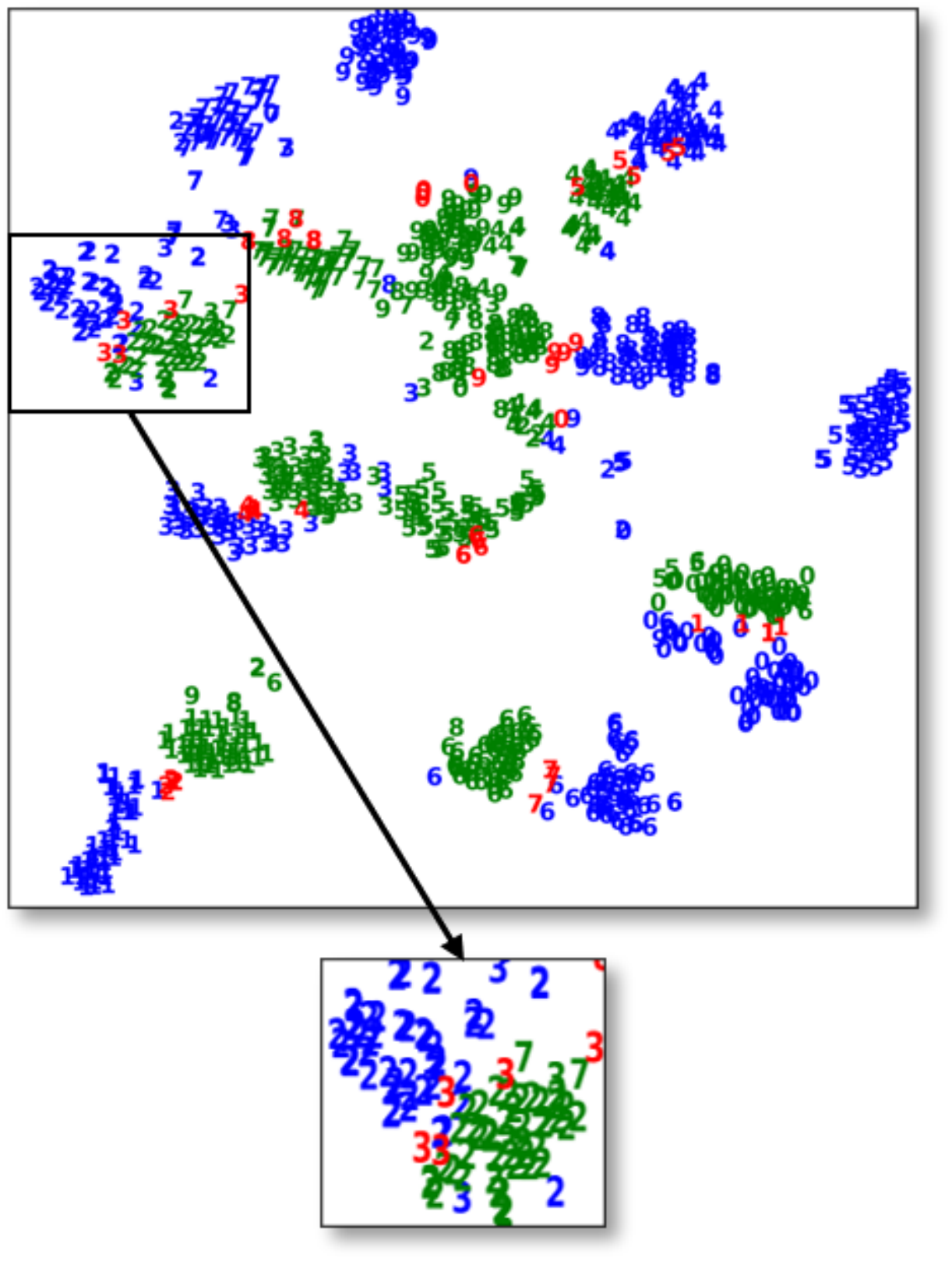}}
  \includegraphics[width=0.3\textwidth]{images/Legend_2.pdf}
  \caption{(Best viewed in color). t-SNE embedding of the data in the representation space (for MNIST $\rightarrow{}$USPS task) learned using DANN, CDAN, MCD, and SSL on source domain data poisoned with watermarked ($\alpha=0.3$) data. Successful poisoning aligns the wrong classes for discriminator-based approaches, as seen in (a) with DANN. Poisoning fails against CDAN because of the pseudo-labels being correct on the target data (b). For MCD, we see 20 distinct clusters highlighting the failure of the method at detecting and aligning target domain data (c). For SSL, the poison data has prevented the correct classes from having very similar representations (d). The failure of most UDA methods with a small amount of watermarked data makes our attack practical and raises serious concerns about the success of these methods.}
  \label{fig:exp_2}
\end{figure}

\section{Bilevel formulation for clean-label attacks}\label{app:clean_label_bilevel}
In this section, we present the bilevel formulation for a clean-label data poisoning attack against UDA methods.
Let $u=\{u_1,...,u_n\}$ denote the poisoned data and $\hat{\mathcal{D}}^\mathrm{val_{target}} = \{(x_i^\mathrm{val_{target}},y_i^\mathrm{val_{target}})\}_{i=1}^{N_\mathrm{val_{target}}}$, be a small set of labeled target domain data accessible to the attacker. 
To ensure a clean label, each poison point $u_i$ must have a bounded perturbation from a base point
$x_i^\mathrm{base}$ i.e, $\|u_i - x_i^\mathrm{base}\| =\|\delta_i\| \leq \epsilon$ and has label of the base i.e., $y_i^\mathrm{base}$. 
Thus, $\hat{\mathcal{D}}^\mathrm{poison} = \{(u_i, y_i^\mathrm{base})\}_{i=1}^{N_{\mathrm{poison}}}$.
The clean-label poison data $u$ is such that when the victim uses $\hat{\mathcal{D}}^\mathrm{source} \bigcup \hat{\mathcal{D}}^\mathrm{poison}$ and $\hat{\mathcal{D}}^\mathrm{target}$ for UDA, 
the accuracy on $\hat{\mathcal{D}}^\mathrm{val_{target}}$ is minimized. 
The bilevel formulation for this attack is as follows:
\begin{equation}
        \max_{u \in \mathcal{U}}\;\; \mathcal{L}(\hat{\mathcal{D}}^\mathrm{val_{target}}; \theta^\ast) \;\;
        \mathrm{s.t.}\;\; \theta^* = \arg\min_{\theta}\; \mathcal{L}_\mathrm{UDA}(\hat{\mathcal{D}}^\mathrm{clean} \bigcup \hat{\mathcal{D}}^\mathrm{poison}, \hat{\mathcal{D}}^\mathrm{target}; \theta).
    \label{eq:bilevel_clean_label}
\end{equation}
The solution to the lower-level problem $\theta^\ast$ are the parameters of the generator and the classifier learned from using a UDA method on the poisoned source domain data and unlabeled target domain data. 
Solving bilevel optimization problems \cite{huang2020metapoison, mehra2019penalty, mehra2020robust} to generate clean-label poison data has previously been shown to be effective. We used an alternating optimization to avoid the computational complexity of solving the bilevel optimization (Eq.~\ref{eq:alternating_clean_label}). However, we believe the attack success can be boosted by solving the bilevel formulation proposed in Eq.~\ref{eq:bilevel_clean_label} and is left for future work.

\section{Additional experiments}\label{app:additional_experiments}
\subsection{Effect of changing the percentage of poison data and divergence between the domains}
We use a simple two moons dataset to illustrate the effect of changing both the poison percentage and the divergence between the two domains. 
The target domain data is the translated version of the source domain data along the x-axis. 
The amount of translation is used to measure the divergence between the source and the target domains.
For this experiment, we use mislabeled target domain data as our poison points.
This is referred to as wrong-label incorrect-domain poisoning in the main paper. The results of this experiment are present in Table~\ref{Table:two_moons}.
Consistent with the results present in rows marked with Poison\_{target} of Table 1 and 2 of the main paper and Table~\ref{Table:exp1_poison_percentage}, we see that increasing the divergence between the source and target domains makes it easy to poison UDA methods with small amounts of poisoned data. Moreover, at a fixed divergence level increases the amount of poisoned data leads to a larger drop in the target domain accuracy.

\begin{table}[tb]
  \caption{Effect of using different percentages of wrong-label incorrect-domain poisoned data on the target domain accuracy when the divergence (D) between the source and target domain is changed. (mean$\pm$s.d. of 5 trials).}
  \label{Table:two_moons}
  \centering
  \small
  \resizebox{0.8\columnwidth}{!}{
    \begin{tabular}{c|c|c|c|c}
    \toprule
    \multirow{2}{*}{Method} &  \multirow{2}{*}{Dataset} & \multicolumn{3}{c}{Target domain accuracy(\%)}  \\
    &  & D=0.25 & D=0.5 & D=0.75 \\
    \midrule
    \multirow{3}{*}{Source Only} & Clean & 99.60$\pm$0.01 & 87.92$\pm$0.96 & 68.4$\pm$0.13 \\
    & Poisoned (5\%) & 79.76$\pm$2.94 & 55.24$\pm$4.42 & 59.44$\pm$2.77\\
    & Poisoned (10\%) & 53.16$\pm$2.44 & 46.52$\pm$3.51 & 42.56$\pm$2.65\\
    \midrule
    \multirow{3}{*}{DANN} & Clean & 98.48$\pm$1.45 & 97.24$\pm$1.27 & 71.80$\pm$3.92\\
    & Poisoned (5\%) & 95.36$\pm$2.54 & 83.08$\pm$1.43 & 62.72$\pm$3.63 \\
    & Poisoned (10\%) & 85.52$\pm$10.1 & 65.56$\pm$8.44 & 40.92$\pm$5.97\\
    \midrule
    \multirow{3}{*}{CDAN} & Clean & 100$\pm$0.0 & 94.52$\pm$1.31 & 76.08$\pm$4.49\\
    & Poisoned (5\%) & 91.36$\pm$2.98 & 77.48$\pm$6.23 & 64.28$\pm$4.95 \\
    & Poisoned (10\%) & 87.12$\pm$1.35 & 73.16$\pm$0.81 & 46.04$\pm$8.73 \\
    \midrule
    \multirow{3}{*}{MCD} & Clean & 100$\pm$0.0 & 91.88$\pm$2.35 & 69.28$\pm$2.77 \\
    & Poisoned (5\%) & 79.68$\pm$5.52 & 70.48$\pm$10.69 & 58.88$\pm$9.96 \\
    & Poisoned (10\%) & 65.81$\pm$2.66 & 44.28$\pm$5.73 & 40.96$\pm$16.4 \\
    \bottomrule
    \end{tabular}
    }
\end{table}

{\bf Effect of poisoning on \cite{long2015learning}:}
Following the original and follow-up works on MMD \cite{combes2020domain}, we also evaluated the effect of our poisoning attacks against this method using the Office-31 dataset. 
The results of poisoning on DAN (Table~\ref{Table:dan}) with 10\% mislabeled target domain data, consistent with the other results shown in Table~\ref{Table:office31_experiment_1} of the main paper, demonstrate the effectiveness of our poisoning attack against this method. 
We used the Pytorch version of the implementation for DAN available from the official code \cite{combes2020domain} to generate these results.

\begin{table}[tb]
  \caption{Decrease in the target domain accuracy for DAN trained on poisoned source domain data (with poisons sampled from the target domain) compared to accuracy attained with clean data on the Office tasks (mean$\pm$s.d. of 3 trials).}
  \label{Table:dan}
  \centering
  \small
  \resizebox{0.8\columnwidth}{!}{
    \begin{tabular}{c|cccccc}
    \toprule
    Dataset & 
    A $\rightarrow{}$ D & A $\rightarrow{}$ W & D $\rightarrow{}$ A 
    & D $\rightarrow{}$ W & W $\rightarrow{}$ A & W $\rightarrow{}$ D\\
    \midrule
    Clean & 82.3 &80.1 &68.9 &98.0 &66.1 &99.0\\
    Poisoned & 59.8 &51.4 &6.7 &53.7 &7.6 &79.9\\
    \bottomrule
    \end{tabular}
    }
\end{table}

{\bf Effect of poisoning on \cite{shu2018dirt}:}
Here we present the results of our poisoning attack against the DIRT-T method proposed by \cite{shu2018dirt}, using the Digits dataset. 
This method uses virtual adversarial training and conditional entropy minimization whose effectiveness is contingent on the quality of the pseudo-labels of the target domain data. 
In the main paper, we presented results of using a similar method, CDAN\cite{long2017conditional}, which relied on the idea of using pseudo-labels for the target domain data in the discriminator.
Consistent with our results of CDAN presented in Table~\ref{Table:digits_experiment_1}, we see our poisoning attacks (Table~\ref{Table:dirt}) are effective at reducing the target domain accuracy obtainable by using the DIRT-T method. 
We used 10\% mislabeled target domain data as our poisons (same as that used in Table~\ref{Table:digits_experiment_1}) and used a similar setting to the official implementation of the paper for our evaluation. 

\begin{table}[tb]
  \caption{Decrease in the target domain accuracy for DIRT-T trained on poisoned source domain data (with poisons sampled from the target domain) compared to accuracy attained with clean data on the Digits tasks (mean$\pm$s.d. of 5 trials). 
  }
  \label{Table:dirt}
  \centering
  \small
  \resizebox{0.9\columnwidth}{!}{
    \begin{tabular}{c|cccc}
    \toprule
    Data & 
    SVHN $\rightarrow{}$ MNIST & MNIST $\rightarrow{}$  MNIST\_M & MNIST $\rightarrow{}$ USPS 
    & USPS $\rightarrow{}$ MNIST \\
    \midrule
    Clean & 92.45$\pm$0.46 & 91.41$\pm$0.16& 98.15$\pm$0.28 & 98.13$\pm$0.09\\
    Poisoned & 0.09$\pm$0.02 &0.37$\pm$0.01 &4.62$\pm$4.33 &0.31$\pm$0.24 \\
    \bottomrule
    \end{tabular}
    }
\end{table}

{\bf Effect of poisoning on \cite{hoffman2018cycada}:}
This work uses a Cycle GAN coupled with a semantic loss dependent on the pseudo-labels for the target domain data, to enforce consistency of the two domains in the pixel space as well as in the representation space. 
When using the poisoned data to train the cycle GAN, we found that it failed to generate good transformations of the source domain. 
However, it could not be concluded whether the failure of cycle GAN was due to poisoning or due to hyperparameters.  
Thus, in our poisoning experiments, we omit the training of the cycle GAN and rely on the transformed version of the source domain images provided in the official repository. 
We treat these transformed images as our original source domain images. 
We present the results of poisoning only the feature level domain adaptation in Table~\ref{Table:cycada} (assuming cycle GAN is trained on clean data). 
The high target accuracy of training using only the transformed source domain data (transformed source only) compared to the target accuracy obtained using feature-level domain adaptation on clean data implies a high similarity between the transformed source and the target domains. 
Due to this high similarity between the transformed source and the target domains, our poisoning attacks (using 10\% mislabeled target domain data as poisoned data) have limited effect. 
For USPS $\rightarrow{}$MNIST, poisoning with 10\% of MNIST training data (6000 images) is comparable to the size of USPS training data (7291 images) and thus we see a significant reduction in the target domain accuracy. 
These results are consistent with the results presented in the main paper, especially on tasks W $\rightarrow{}$D and D $\rightarrow{}$W in Table~\ref{Table:office31_experiment_1}, where poisoning fails since the source and target domains are very similar (indicated by high source only performance). 

\begin{table}[tb]
  \caption{Decrease in the target domain accuracy for CYCADA trained on poisoned source domain data (with poisons sampled from the target domain) compared to accuracy attained with clean data on the Digits tasks (mean$\pm$s.d. of 5 trials). 
  }
  \label{Table:cycada}
  \centering
  \small
  \resizebox{0.9\columnwidth}{!}{
    \begin{tabular}{c|cccc}
    \toprule
    Method & 
    SVHN $\rightarrow{}$ MNIST & MNIST $\rightarrow{}$ USPS 
    & USPS $\rightarrow{}$ MNIST \\
    \midrule
    Transformed Source Only (Clean data) &
74.5$\pm$0.3 &
95.6$\pm$0.2 &
96.4$\pm$0.1 \\
\midrule
Feature level (Clean data) &
90.4$\pm$0.4 &
95.6$\pm$0.2 &
96.5$\pm$0.1 \\

Feature level (Poisoned data) &
84.3$\pm$2.3 &
93.3$\pm$0.5 &
60.7$\pm$3.5 \\
    \bottomrule
    \end{tabular}
    }
\end{table}

\section{Additional related work}\label{app:additional_related_work}
Previous work \cite{zhao2019learning}, presented an information-theoretic lower bound to explain the failure of learning a domain invariant representation when marginal label distributions differ across the source and target domains. 
In comparison, our lower bound does not require any assumption on the data distributions. Our bound presents a necessary condition for successful learning in the UDA setting. 
In particular, our lower bound shows that a UDA method may succeed or fail at target generalization i.e. the term $max \ (e_S( \tilde{f}_S, \tilde{f}_T ), \ e_T ( \tilde{f}_S, \tilde{f}_T ))$ in our lower bound (Eq.~\ref{eq:lower bound}) can be small or large, even when a domain invariant representation is learnt that minimizes source error.

Unlike the bounds presented by previous works \cite{zhao2019learning,johansson2019support} which fail to provide insights into the behavior of UDA when marginal label distributions are the same across the two domains, our bound remains tight.  
In particular, the information-theoretic lower bound of \cite{zhao2019learning} (Theorem 4.3 of \cite{zhao2019learning}) becomes vacuous when $d_{JS}(p^Y_S, p^Y_T) = 0$. 
However, our lower bound will be non-vacuous in this scenario as long as UDA methods are used i.e., $e_S$ and $D_1(\tilde{p}_S, \tilde{p}_T)$ are minimized.
A concrete example where our bound is tighter than the bound of Theorem 4.3 of \cite{zhao2019learning} is their example in section 4.1 where the true target risk is 1. 
As discussed in their paper (last paragraph below Corollary 4.1), their lower bound on the example is vacuous and says that the target risk will be greater than 0. 
In comparison, our lower bound in Theorem 1 says that the target risk is greater or equal to 1 which is tighter and much more informative. 

Compared to the previous work of \cite{wu2020representation} which decomposed the target risk into source risk, representation conditional label divergence, and representation covariate shift and explained the reason for the failure of DANN to be its inability to account for representation conditional label divergence. 
Our lower bound suggests a similar explanation for the failure of DANN i.e, without access to any labeled target domain data, DANN may learn a representation that induces labeling functions $(\tilde{f}_S)$ and $(\tilde{f}_T$ which do not agree with each other on the source and target domains. 
This difference in the induced labeling functions is captured by $max (e_S( \tilde{f}_S, \tilde{f}_T ), \ e_T ( \tilde{f}_S, \tilde{f}_T ))$ term in our lower bound. 
Our explanation for UDA failure also extends beyond DANN as we illustrated through our data poisoning attacks, which are concrete ways to lead UDA algorithms to learn a representation that incurs high error on the target domain.

\section{Details of the experiments}\label{app:experiments_details}
All codes are written in Python using Tensorflow/Keras and were run on Intel Xeon(R) W-2123 CPU with 64 GB of RAM and dual NVIDIA TITAN RTX. Dataset details and model architectures used are described below.

\subsection{Dataset description}
Here we describe the details of the datasets used for the Digits and Office-31 tasks.
\\
\\
{\bf Digits:} 
For this task, we use 4 datasets: MNIST, MNIST\_M, SVHN, and USPS. 
We evaluate four popular tasks under this, namely, SVHN $\rightarrow{}$MNIST, MNIST $\rightarrow{}$MNIST\_M, MNIST $\rightarrow{}$USPS and USPS $\rightarrow{}$MNIST. For SVHN $\rightarrow{}$MNIST, we train on 73,257 images from SVHN and 60,000 images from MNIST while testing on 10,000 MNIST images. For MNIST $\rightarrow{}$MNIST\_M, we use 60,000 from MNIST and MNIST\_M for training and test on 10,000 MNIST images. Lastly, for MNIST $\rightarrow{}$USPS and USPS $\rightarrow{}$MNIST, we use 2,000 images from MNIST and 1,800 images from USPS for training. We test on the 10,000 MNIST images and 1,860 USPS images.
\\\\
{\bf Office-31:}
The dataset contains a total of 4110 images belonging to 31 categories from 3 domains: Amazon (A), DSLR(D), and Webcam(W). We evaluate the performance of UDA on all six tasks, namely, A $\rightarrow{}$D, A $\rightarrow{}$W, D $\rightarrow{}$A, D $\rightarrow{}$W, W $\rightarrow{}$A, W $\rightarrow{}$D.\\

\subsection{Model architecture}
Here we describe the model architectures used for different tasks. To fairly compare the performance of different UDA methods and eliminate the effect of architecture changes in improving the performance of different methods, we make use of similar model architectures for different methods, as described below. The effectiveness of these architectures has also been shown by previous works.

{\bf Digits:}
The architectures used for MNIST $\rightarrow{}$MNIST\_M, MNIST $\rightarrow{}$USPS and USPS $\rightarrow{}$MNIST involves a shared convolution neural network. The output of this shared network is fed into a softmax classifier and the discriminator. The architecture of the shared network consists of a convolution layer with a kernel size of 5x5, 20 filters, and ReLU activation, followed by a max-pooling layer of size 2x2. This is followed by another convolution layer with a 5x5 kernel, 50 filters, and ReLU activation followed by similar max pooling and a dropout. Then we have a fully connected layer with ReLU activation of size 500 followed by a dropout layer. For the discriminator, we use two dense layers with 500 units each followed by a ReLU and a dropout layer. This is followed by a 2 unit softmax layer.
For MCD, we use the following architecture for the generator on MNIST $\rightarrow{}$MNIST\_M task. A convolution layer with a kernel size of 5x5, 32 filters, and ReLU activation, followed by a max-pooling layer of size 2x2. This is followed by another convolution layer with a 5x5 kernel, 48 filters, and ReLU activation followed by a similar max-pooling layer. We use 2 dense layers for the classifier with 100 units followed by ReLU activation and dropout layers. This is followed by the softmax layer. Unlike the original work MCD\cite{saito2018maximum}, we do not use batch normalization layers in these tasks to make architectures consistent across different methods.  
\begin{wrapfigure}[10]{r}{0.4\textwidth}
\vspace{-0.1cm}
\includegraphics[width=0.4\columnwidth]{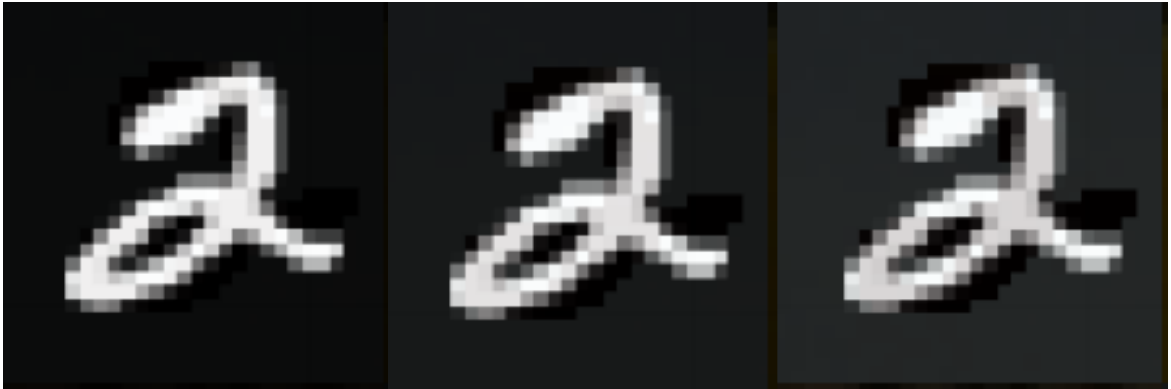}
\vspace{-0.33cm}
\caption{Watermarked poison data for MNIST $\rightarrow{}$ MNIST\_M task with $\alpha$ in \{0.05, 0.10, 0.15\}.}
\label{fig:watermarking_poison}
\end{wrapfigure}
For SVHN $\rightarrow{}$MNIST we use the following architecture for the generator. A convolution layer with a kernel size of 5x5, 64 filters, the stride of 2 followed by batch normalization, dropout, and ReLU activation layer. This is followed by another convolution layer with a kernel size of 5x5, 128 filters, the stride of 2 followed by batch normalization, dropout, and ReLU activation layer. Then another convolution layer with a kernel size of 5x5, 256 filters, the stride of 2 followed by batch normalization, dropout, and ReLU activation layer. This is followed by a dense layer with 512 units followed by batch normalization, ReLU activation, and a dropout layer. We use the softmax layer for classification. For the discriminator, we use two dense layers with 500 units each followed by a ReLU and a dropout layer. This is followed by a 2 unit softmax layer.  For MCD, we use the same architecture for the generator except that we use max-pooling instead of convolution layers with stride 2 to downsample the representation. The classifier uses the output of the generator and feeds into a dense layer with 256 units followed by batch normalization and ReLU activation layers. This is followed by a softmax layer.

{\bf Office-31:} 
For office experiments, we use the publicly available code of the work\footnote{\url{https://bit.ly/34EFb52}} \cite{combes2020domain} and supply the poisoned data by adding them to the input files being used by the code. We use all default options of the code and use DAN, CDAN, IW-DAN, IW-CDAN algorithms. This is done to eliminate the effect of hyperparameters on the performance of the UDA algorithms on the Office-31 dataset and be able to fairly compare the performance of poisoning. To obtain the representation trained only on the source domain data, we initialize a ResNet50 model with weight pre-trained on Imagenet. We then update the representation by training on respective source domain data for different tasks.

\begin{figure}[tb]
  \centering
    \subfigure[Base data chosen from the source domain]{\includegraphics[width=0.45\columnwidth, height=2.5cm]{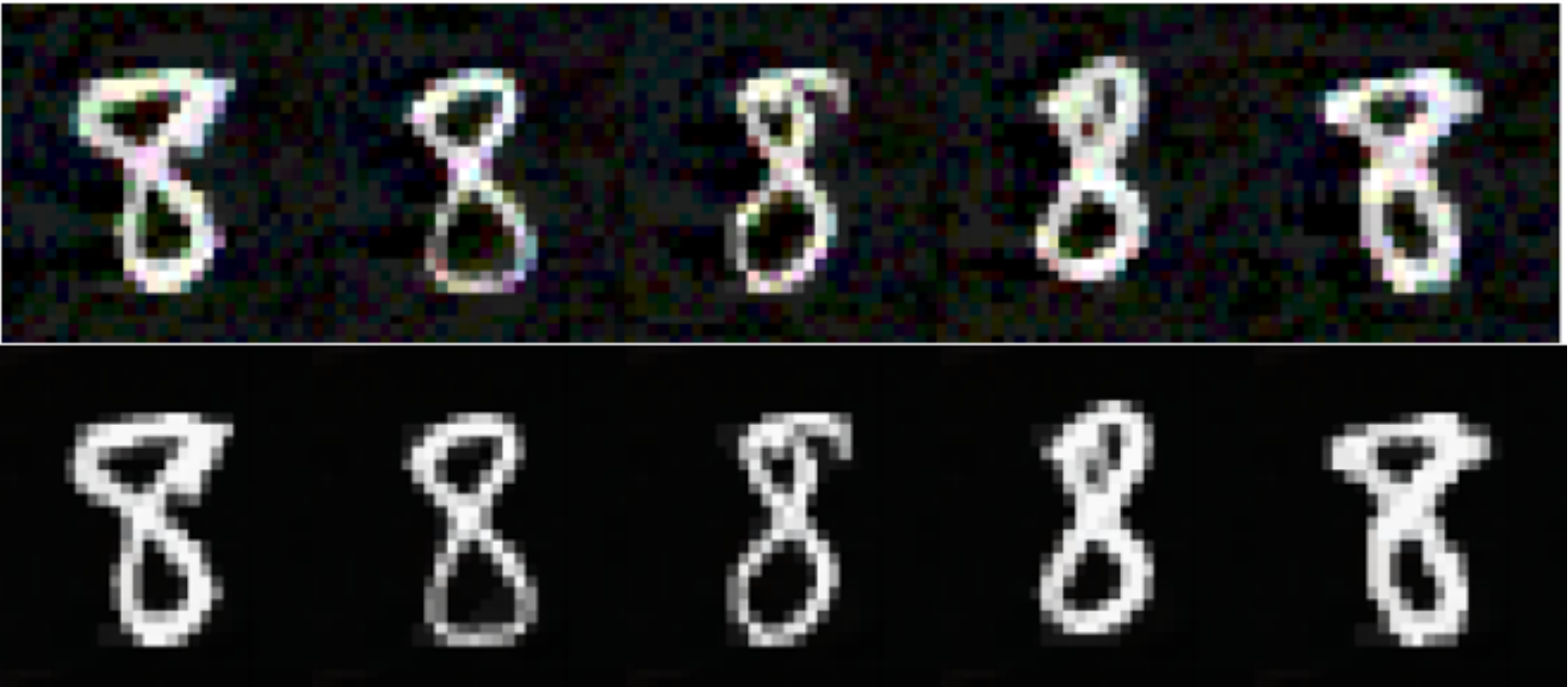}}
    \subfigure[Base data chosen from the target domain]{\includegraphics[width=0.45\columnwidth, height=2.5cm]{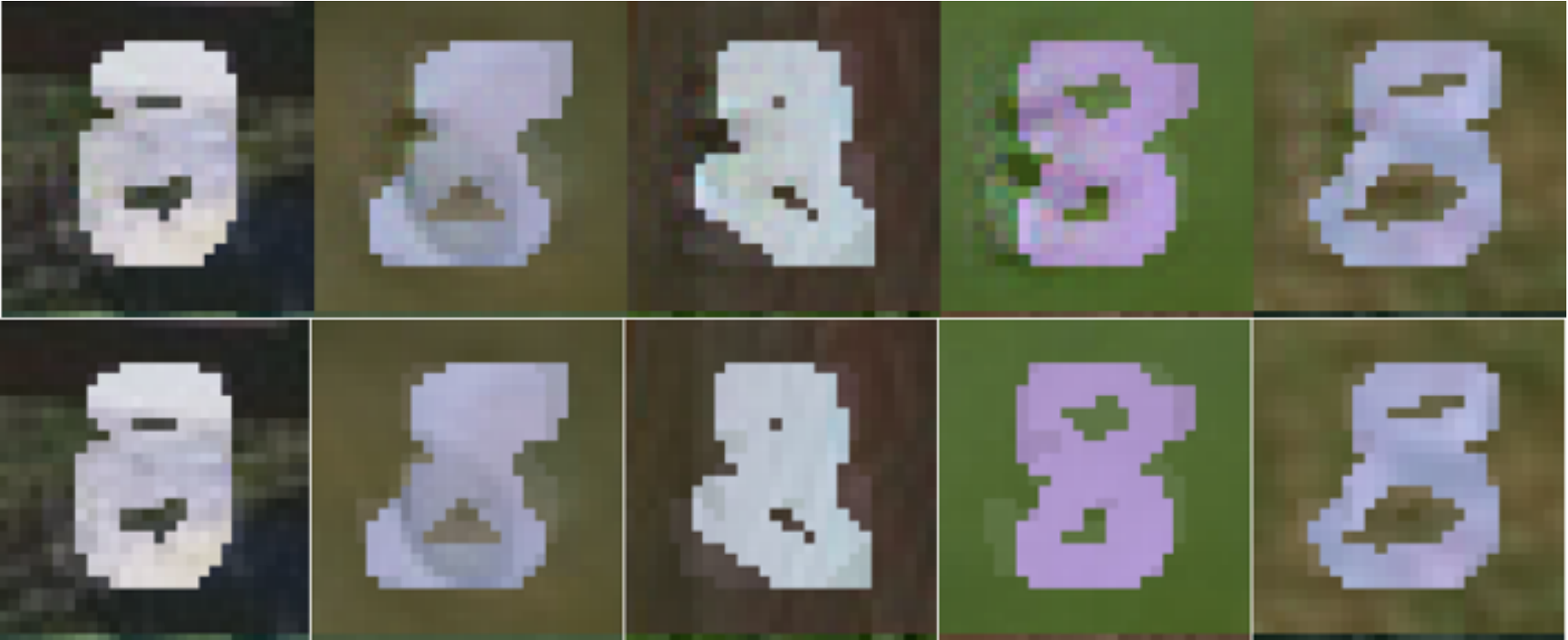}}  
  \caption{Poison data (top rows) obtained after solving Eq.~\ref{eq:alternating_clean_label} by using DANN as the UDA method, with base data (bottom rows) initialized from the source domain (left) and the target domain (right). Attack success with poison data initialized from the target is significantly higher than the attack success obtained with poison data initialized from the source, from under the same maximum permissible distortion constraint ($\epsilon=0.1$ in $\ell_\infty$ norm) as seen in Fig.~\ref{fig:clean_label_poison}. 
  }
  \vspace{-0.5cm}
  \label{fig:clean_label_poison_data}
\end{figure}

\subsection{Clean-label attack on MNIST $\rightarrow{}$MNIST\_M}
For this experiment, 1\% poison data is used to prevent the alignment of a target test point to its correct class. We test the attack on the binary classification problems (3 vs 8). Two approaches to initializing the poison data are evaluated. In the first approach, the poison data is initialized from the source domain data, and in the second approach, it is initialized from the target domain data. In both cases, the poison is picked from the class opposite to the true class of the target test point. Moreover, the poison data is initialized using the points closest in the input space to the target test point. The poison data obtained after solving Eq.~\ref{eq:alternating_clean_label} is added to the source domain data and UDA methods are retrained from scratch. The attack is considered successful if the target test point is misclassified after this retraining. For the results shown in Fig.~\ref{fig:clean_label_poison}, we randomly targeted 20 points and obtained poison data corresponding to each UDA method. Attack success is reported after evaluating UDA methods on five random initializations by adding the generated poison data in the source domain. To control the amount of maximum distortion between experiments, we add a constraint on the maximum permissible distortion to poison data using $\ell_{\infty}$ norm and use a value of $\epsilon=0.1$. The poison data obtained after solving the optimization with base data chosen from the source and target domains with DANN as the UDA method are shown in Fig.~\ref{fig:clean_label_poison_data}. To generate poison data that remains effective even after UDA methods are trained from scratch, we make use of multiple randomly initialized networks during poison generation. Following the work \cite{huang2020metapoison}, we reinitialize the models at different points during optimization. This re-initialization scheme helps train UDA methods with different random initializations and for a different number of epochs making the poison data more resilient to initialization change that can happen at test-time.

\end{document}